\pdfoutput=1

\documentclass[11pt]{article}

\usepackage{EMNLP2023}

\usepackage{times}
\usepackage{latexsym}
\usepackage{tikz}

\usepackage[T1]{fontenc}

\usepackage[utf8]{inputenc}

\usepackage{microtype}

\usepackage{inconsolata}

\usepackage{subcaption}
\usepackage{graphicx}

\definecolor{mygreen}{RGB}{102,194,166}
\definecolor{myorange}{RGB}{252,141,89}
\definecolor{myblue}{RGB}{141,160,203}

\usepackage{ amssymb, amsthm, amsmath }

\usepackage{mathtools}
\usepackage{multirow}
\usepackage[linesnumbered,ruled]{algorithm2e}
\usepackage{adjustbox}
\usepackage[linesnumbered,ruled]{algorithm2e}
\newtheorem{theorem}{Theorem}

\newtheorem{definition}{Definition}

\DeclarePairedDelimiterX{\inp}[2]{\langle}{\rangle}{#1, #2}
\DeclarePairedDelimiter\abs{\lvert}{\rvert}%

\usepackage{color}

\usepackage{listings}
\title{Locally Differentially Private Document Generation  Using  \\ Zero Shot Prompting}

\author{Saiteja Utpala \\
  Cohere For AI \\
  \texttt{saitejautpala@gmail.com} \\\And
  Sara Hooker \\
  Cohere For AI \\
  \texttt{sarahooker@cohere.com} \\\And
  Pin Yu Chen \\
  IBM Research \\
  \texttt{pin-yu.chen@ibm.com} \\ }

\begin{document}
\maketitle

\begin{abstract}

Numerous studies have highlighted the privacy risks associated with pretrained large language models. In contrast, our research offers a unique perspective by demonstrating that pretrained large language models can effectively contribute to privacy preservation. We propose a locally differentially private mechanism called DP-Prompt, which leverages the power of pretrained large language models and zero-shot prompting to counter author de-anonymization attacks while minimizing the impact on downstream utility. When DP-Prompt is used with a powerful language model like ChatGPT (gpt-3.5), we observe a notable reduction in the success rate of de-anonymization attacks, showing that it surpasses existing approaches by a considerable margin despite its simpler design. For instance, in the case of the IMDB dataset, DP-Prompt (with ChatGPT) perfectly recovers the clean sentiment F1 score while achieving a 46\% reduction in author identification F1 score against static attackers and a 26\% reduction against adaptive attackers. We conduct extensive experiments across six open-source large language models, ranging up to 7 billion parameters, to analyze various effects of the privacy-utility tradeoff. Code is avaliable at \url{https://github.com/SaitejaUtpala/dp_prompt}

\end{abstract}

\section{Introduction}

The vast amount of online text data has the potential to reveal numerous user attributes, making individuals easily identifiable \citep{rao2000can, hovy2015user, preoctiuc2015analysis}. While private information can be directly disclosed through specific phrases in the text, it can also be implicitly inferred. For instance, linguistic patterns embedded within the text can inadvertently facilitate authorship attribution \cite{kevselj2003n, shrestha2017convolutional}, leading to unintended privacy leakage.

    
    

\begin{figure}[t]
\includegraphics[width=\linewidth]{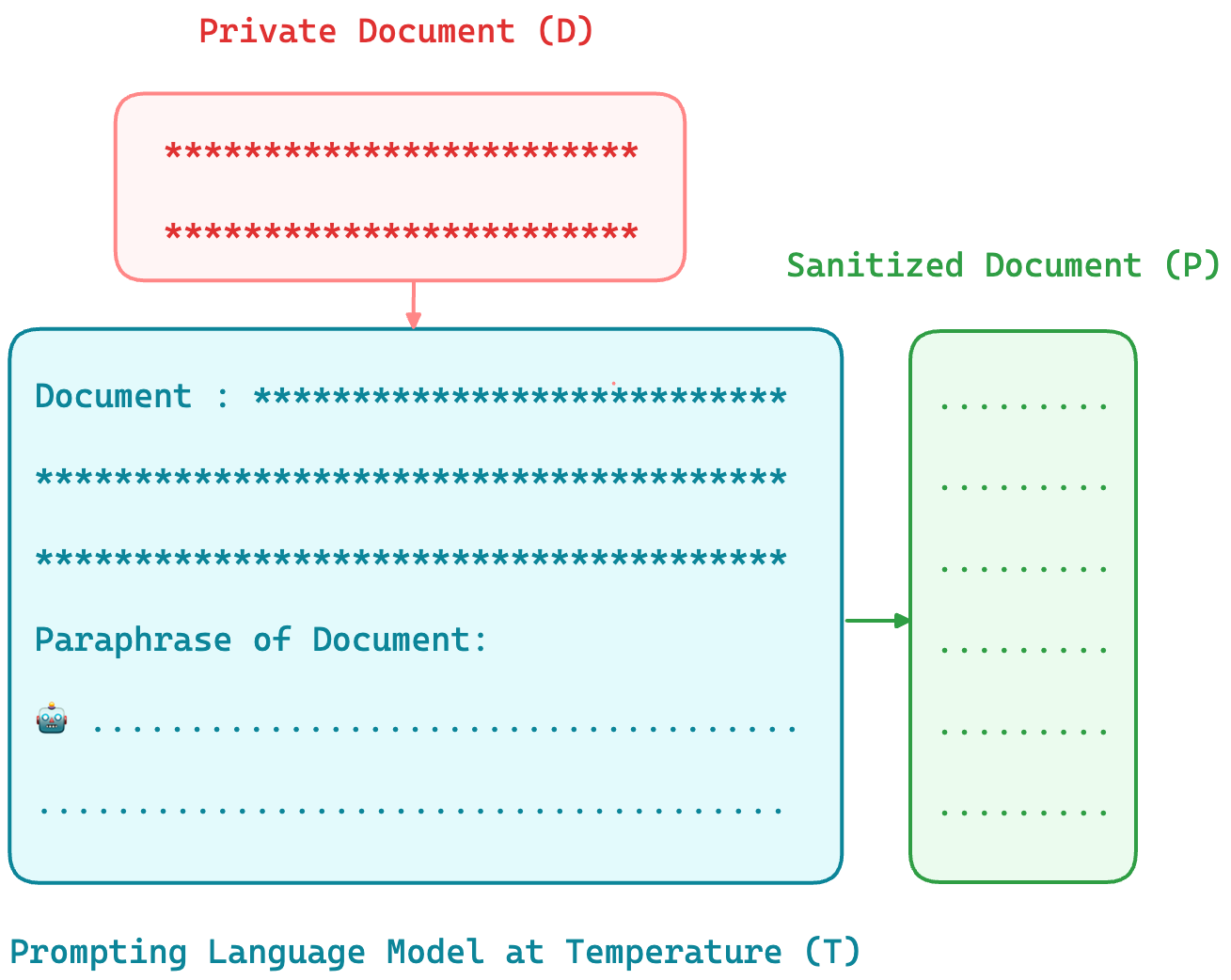}
\caption{
Overview of the proposed DP-Prompt mechanism. Given a private document (D), DP-Prompt generates a sanitized document (P) while ensuring local differential privacy. The level of privacy guarantee is controlled by adjusting the sampling temperature (T) during the decoding process.}
\label{fig:DP-Prompt}
\end{figure}

    
    


An illustrative real-world scenario is the AOL search data leak in August 2006 \cite{pass2006picture}. The incident unfolded when AOL mistakenly released detailed search logs of their users, wrongly assuming that the data had been adequately anonymized through the use of random user IDs. Unfortunately, the released logs contained sufficient personally identifiable information, leading to the identification of numerous individuals \cite{timesaol,jones2007know}. This breach of privacy triggered widespread public outcry and led to the initiation of class action lawsuits.

This case is just one among many that highlights the limitations of ad-hoc privacy approaches that may give the impression of providing privacy but ultimately fall short.  Differential privacy ({DP}) provides a rigorous treatment for the notion of data privacy by providing plausible deniability by precisely quantifying the deviation in the model's output distribution under modification of a small number of data points \citep{dwork2006calibrating, dwork2014algorithmic}. The provable guarantees offered by DP, coupled with its compelling properties such as immunity to arbitrary post-processing and graceful composability, have established it as the de facto standard for privacy. DP has witnessed widespread adoption and numerous deployments in both private \citep{erlingsson2014rappor, apple2017learning, near2018differential} and public organizations \cite{abowd2018us}. 


\begin{figure}[t]
    \centering
    \begin{minipage}{\columnwidth}
        \centering
        \begin{subfigure}{0.49\linewidth}
            \centering
            \includegraphics[width=\linewidth]{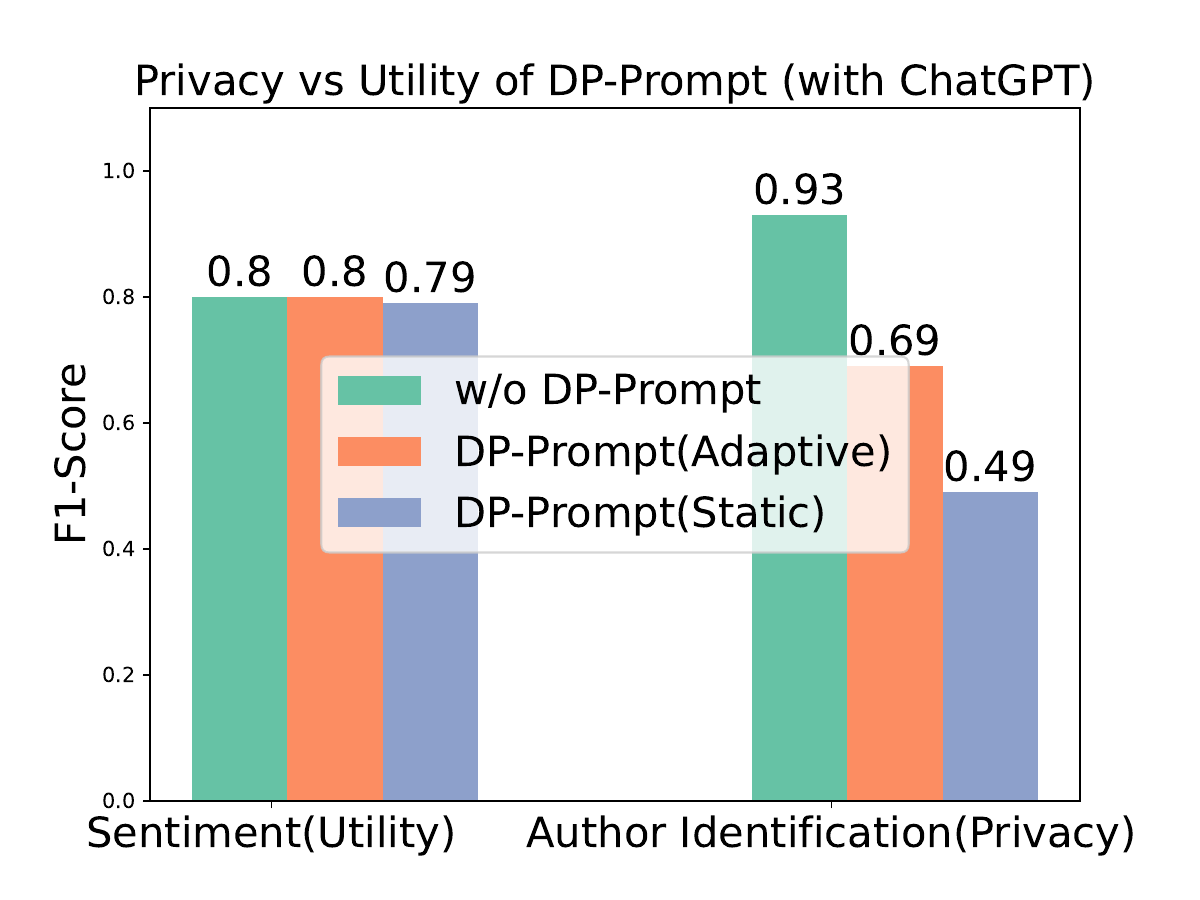}
            \subcaption{IMDB (embedding access)}
        \end{subfigure}
        \begin{subfigure}{0.49\linewidth}
            \centering
            \includegraphics[width=\linewidth]{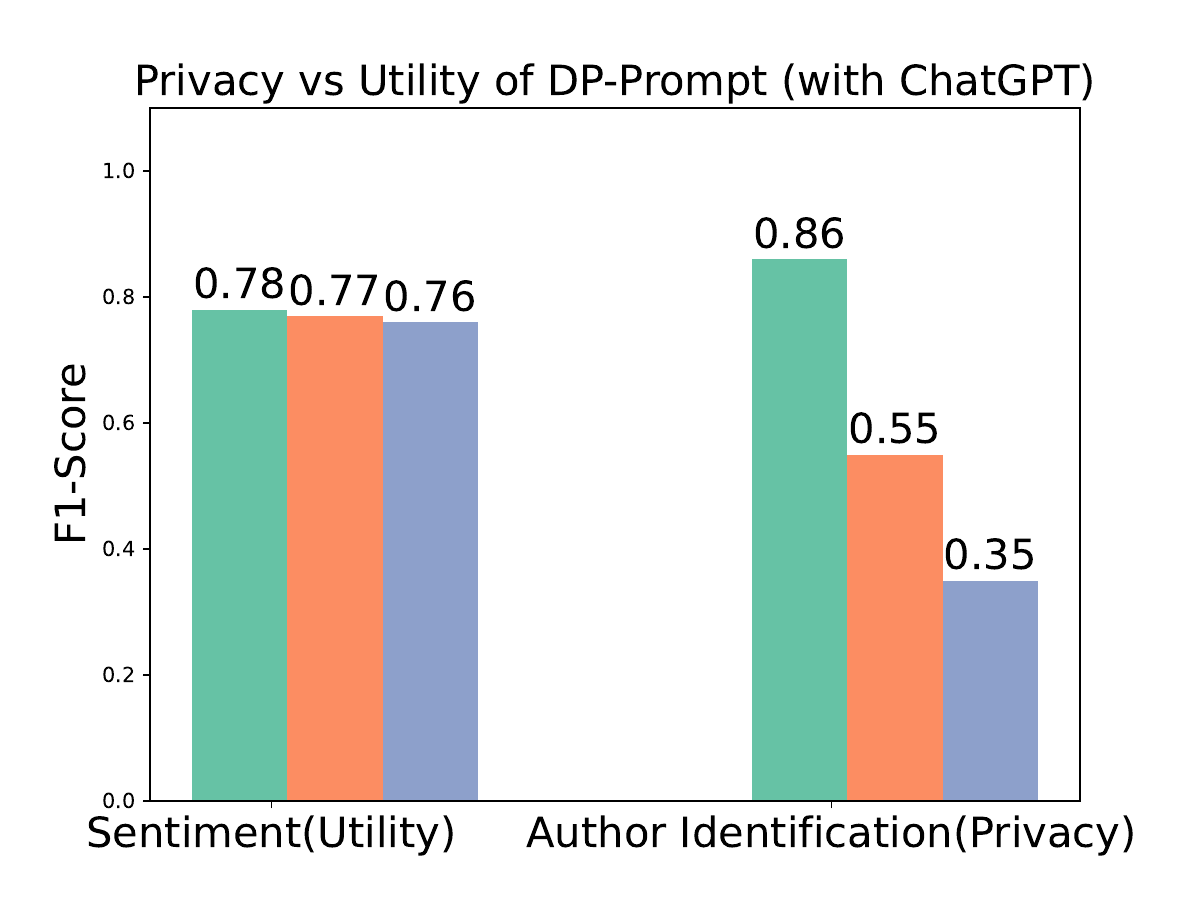}
            \subcaption{Yelp (embedding access)}
        \end{subfigure}
        \begin{subfigure}{0.49\linewidth}
            \centering
            \includegraphics[width=\linewidth]{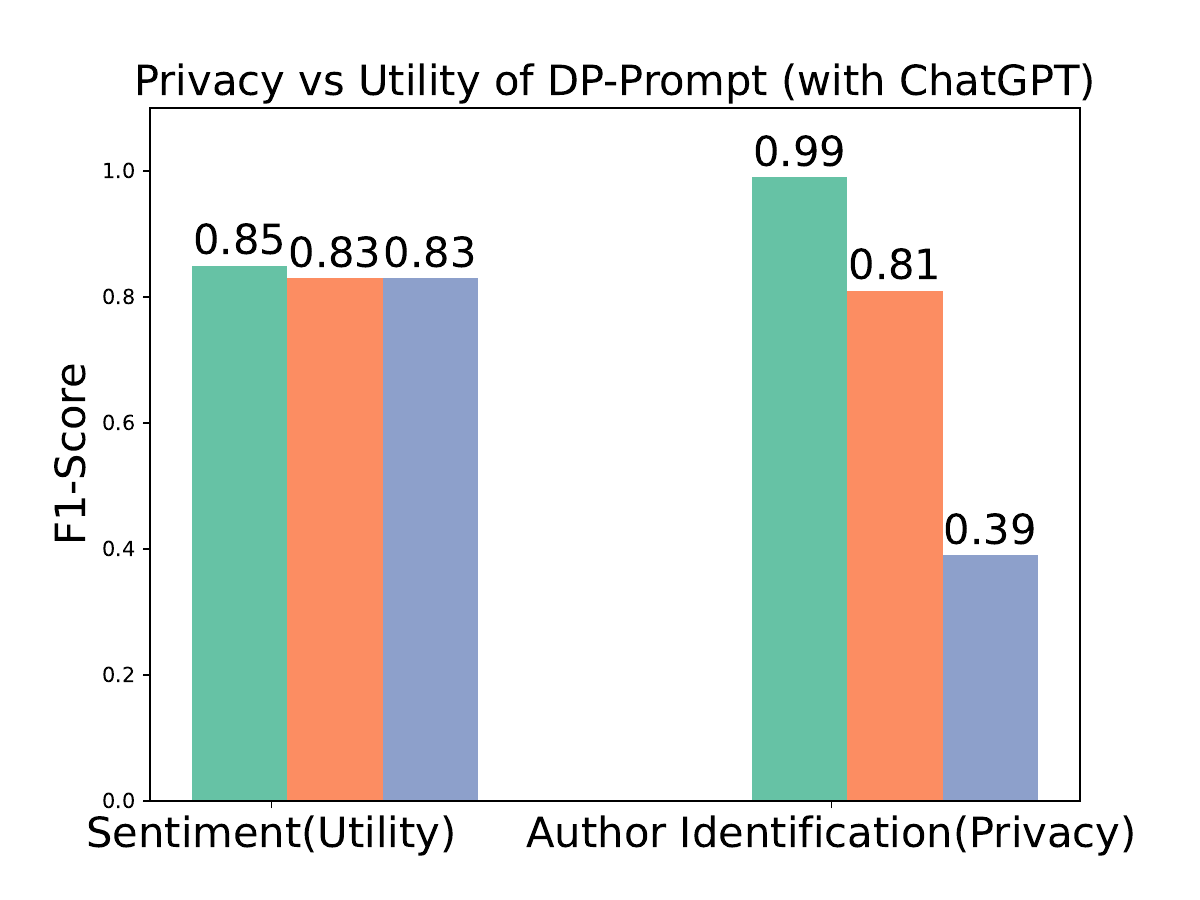}
            \subcaption{IMDB (text access)}
        \end{subfigure}
        \begin{subfigure}{0.49\linewidth}
            \centering
            \includegraphics[width=\linewidth]{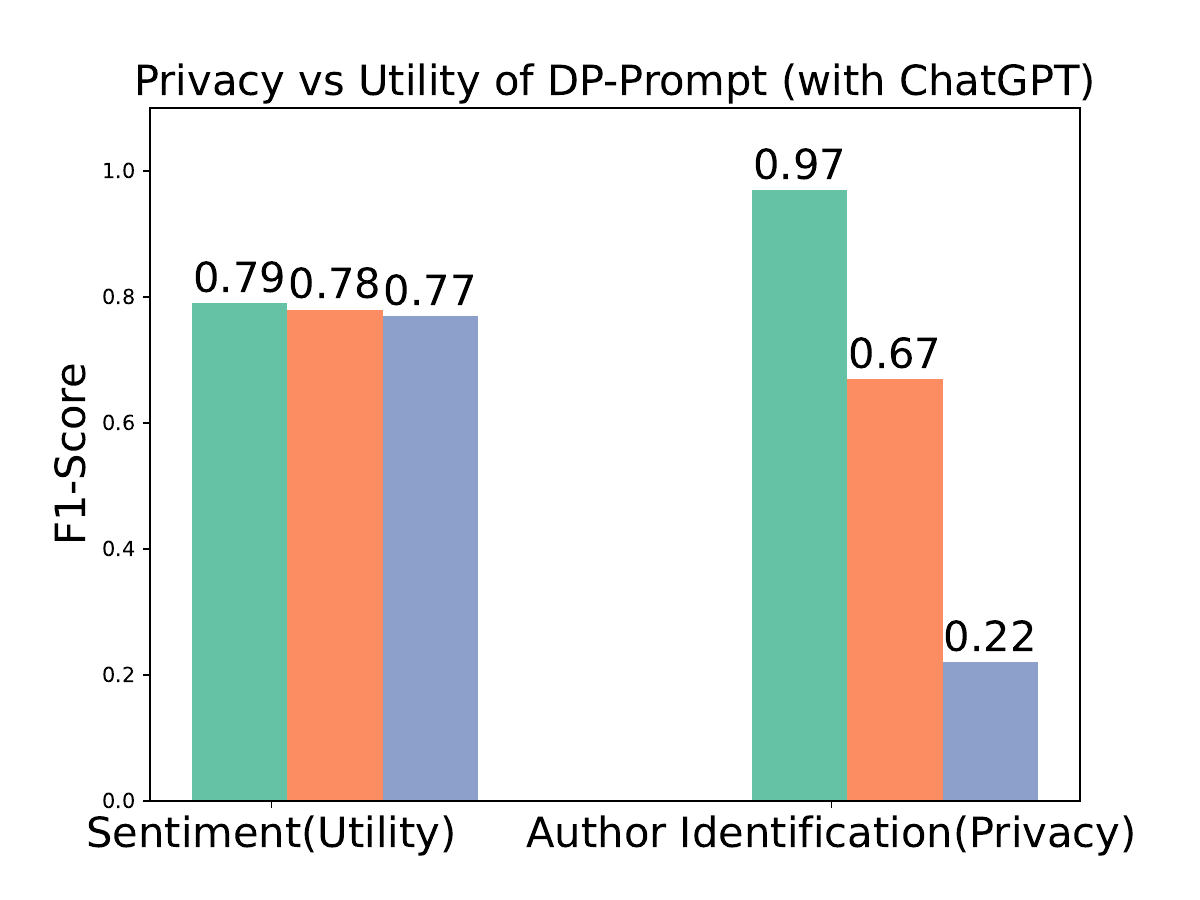}
            \subcaption{Yelp (text acccess)}
        \end{subfigure}
                \caption{Overview of the privacy-utility tradeoff with DP-Prompt (using ChatGPT) for the IMDB and Yelp datasets, conducted at a temperature of $1.5$. The terms 'Static' and 'Adaptive' refer to the attack models defined in Definition \ref{def:attack_models}.}
        \label{fig:dp_prompt_tradeoff}
    \end{minipage}
\end{figure}

To address the issue of deanonymization attacks, various approaches have been proposed within the DP framework. These approaches encompass word-level strategies \cite{feyisetan2020privacy, xu2020differentially, carvalho2021tem}  where noise is added at the word level, as well as sentence-level techniques \cite{meehan2022sentence} where noise is added at the sentence level. However, recent research by \cite{mattern2022limits} has identified limitations in word-level approaches, particularly their disregard for contextual information. To overcome these limitations, Mattern introduced a mechanism that fine-tunes the GPT-2 model \cite{radford2019language} specifically for paraphrasing tasks, resulting in the generation of sanitized versions of documents. While promising, the approach is limited by their reliance on annotated paraphrasing data, extensive computing resources for larger models, and the quality of annotations.


We propose \textbf{DP-Prompt}, a novel and straightforward solution to address deanonymization attacks. Our method leverages pretrained large language models by directly prompting them to generate paraphrases. These paraphrases are then released as sanitized documents in a zero-shot manner (See Figure \ref{fig:DP-Prompt}). Our motivation for this approach stems from two important factors. Firstly, recent research \cite{bevendorff2019heuristic, mattern2022limits} has shown that paraphrasing is a robust defense mechanism against deanonymization attacks. Secondly, growing evidence suggests that pretrained large language models can effectively tackle complex tasks without the need for task-specific and expensive fine-tuning \cite{brown2020language, chowdhery2022palm, chung2022scaling, kojimalarge, gpt4technicalreport}, through zero-shot prompting. 


By harnessing the  capabilities of pretrained large language models, DP-Prompt offers a straightforward and powerful solution to mitigate the risk of deanonymization. It provides a promising alternative that can be widely applicable, particularly in the context of on-device large language models where text completion tasks require significantly fewer resources. We summarize the contributions as follows: 



\begin{itemize}
    \item We propose DP-Prompt, a new, simple, and computationally effective differentially private (DP) mechanism designed as a defense against de-anonymization attacks. DP-Prompt takes a private document and generates a paraphrased version using zero-shot prompting. The resulting paraphrased document is then released as a sanitized document, as illustrated in Figure \ref{fig:DP-Prompt}.
    \item We demonstrate that DP-Prompt, in combination with ChatGPT (gpt-3.5), surpasses all current methods in terms of utility for any level of privacy. Our approach successfully recovers clean sentiment F1 score while significantly reducing the accuracy of author de-anonymization attacks. Refer to Figure \ref{fig:dp_prompt_tradeoff} for an overview of these results.
    
    \item To demonstrate the broad applicability of DP-prompt,  We conduct extensive experiments  with $6$ open source models ranging upto $7$ billion parameters to study the privacy-utiliy tradeoff.  
\end{itemize}



\begin{table*}[htbp!]
\centering
\resizebox{0.80\textwidth}{!}{%
 \begin{tabular}{c c c c}
 \hline
 Mechanism & Privacy Level   & Requires fine-tuning &  Generates sanitized doc  \\ [0.5ex] 
 \hline
  Madlib \citep{feyisetan2020privacy} &  Word level Metric-DP  & No & Yes  \\
  Mahanolbis \cite{xu2020differentially} &  Word level Metric-DP & No & Yes   \\
  TEM \citep{carvalho2021tem} &  Word level Metric-DP  & No & Yes  \\
 Truncated Laplace \citep{meehan2022sentence} & Sentence level Pure-DP  & No &  No     \\ 
 Deep Candidate \citep{meehan2022sentence} & Sentence level Pure-DP  & Yes & No     \\ [0.1ex]
 Paraphraser \citep{mattern2022limits} & Document level Pure-LDP  & Yes & Yes \\
 \hline 
 DP Prompt  (Ours) & Document level Pure-LDP & No  & Yes  \\ [0.1ex]
 \hline
 \end{tabular}}
 \vspace{0.3cm}
 \caption{We compare our proposed method, DP-Prompt, with related work on various factors. The "Privacy level" indicates the privacy guarantee provided by each mechanism. "Fine-tuning" denotes whether the mechanism involves fine-tuning a model as an intermediate step. The last column, "Generates sanitized doc," indicates whether the mechanism can output a fully sanitized document instead of just sanitized embeddings.}
  \label{tab:summary}
  \vspace{-2mm}
\end{table*}

\section{Preliminaries}

A mechanism $\mathcal{M} : \mathcal{D} \rightarrow \mathcal{V}$ achieves $\epsilon$-PureDP if, for all inputs $\textup{D}, \textup{D}' \in \mathcal{D}$ that differ in one element, and for all $V \subseteq \text{Range}(\mathcal{M})$, $\textup{Pr} \left[ \mathcal{M}(\textup{D}) \in V \right] \leq \exp{(\epsilon)} \textup{Pr} \left[ \mathcal{M}(\textup{D}') \in V \right]$ \cite{dwork2006calibrating}.

Metric Differential Privacy (Metric-DP) \cite{andres2013geo, chatzikokolakis2013broadening} is a relaxation of Pure-DP that applies to data represented in a general metric space. For a given distance metric $d : \mathcal{D} \times \mathcal{D} \rightarrow \mathbb{R}_{+}$, a mechanism $\mathcal{M} : \mathcal{D} \rightarrow \mathcal{V}$ achieves $\epsilon d$-MetricDP if, for any $\textup{D}, \textup{D}' \in \mathcal{D}$ and for all $V \subseteq \text{Range}(\mathcal{M})$, $\textup{Pr} \left[ \mathcal{M}(\textup{D}) \in V \right] \leq \exp{(d(\textup{D}, \textup{D}'))} \textup{Pr} \left[ \mathcal{M}(\textup{D}') \in V \right]$.


Local differential privacy (LDP) \citep{kasiviswanathan2011can, duchi2013local, xiong2020comprehensive} is a privacy framework where data is locally perturbed before transmission, considering the presence of an untrusted data collector or server. The formal definition of LDP is as follows:

\begin{definition}[PureLDP]\label{def:pure-ldp}
    A randomized mechanism $\mathcal{M} : \mathcal{D} \rightarrow \mathcal{V}$ is said to be $\epsilon$-\textup{PureLDP} if for any pair of inputs $\textup{D}, \textup{D}' \in \mathcal{D}$ and for all $V \subseteq \textup{Range}(\mathcal{M})$
     \begin{align*}
       \textup{Pr}[ \mathcal{M}(\textup{D}) \in V  ] \leq \exp{(\epsilon)} \textup{Pr}  [ \mathcal{M}(\textup{D}') \in V  ].
   \end{align*}
\end{definition}




There is a growing consensus that, despite the assurance of formal guarantees, it is imperative to subject differentially private mechanisms to robust privacy attacks that simulate strong and malicious adversaries \citep{jayaraman2019evaluating, blanco2022critical}. Such evaluation allows to effectively assess the empirical privacy provided by the mechanism in real-world scenario. To this end we define four attack models depending its adaptivity and mode of access.


\begin{figure*}[htbp!]
\centering
  \includegraphics[width=0.77\textwidth]{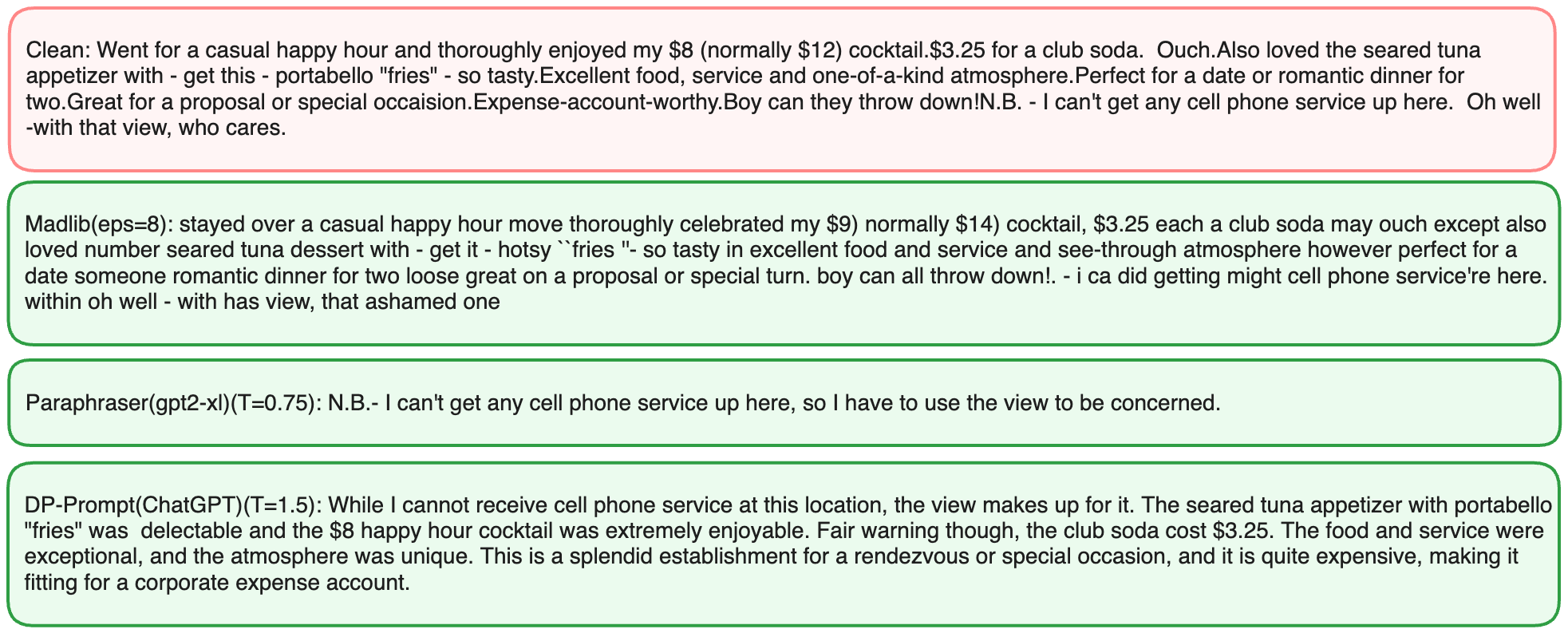}
  \caption{Sample illustration of clean and sanitized documents for various mechanisms. For ChatGPT, prompt is "Review:[review]Paraphrase of the Review:" where review is the clean review.}
\end{figure*}

\begin{definition}[Attack Models]
\label{def:attack_models}
   Consider a collection of private documents $(D_1, \dots, D_n)$ from distribution $\mathcal{D}$ with associated author identities $(a_1, \dots, a_n)$ and embeddings $(E_1, \dots, E_n) \sim \mathcal{E}$. 

For text-to-text sanitization using mechanism $\mathcal{M}_{\textup{text}}$, the sanitized documents are represented as $(P_1, \dots, P_n) \sim \mathcal{P}_{\mathcal{M_{\textup{text}}}}$. For text-to-embedding sanitization via mechanism $\mathcal{M}_{\textup{embedding}}$, the sanitized embeddings are denoted as $(N_1, \dots, N_n) \sim \mathcal{N}_{\mathcal{M_{\textup{embedding}}}}$

    \begin{itemize}
    \item \textbf{\textup{Static Attacker with Embedding Access}}: Has access to clean documents $(D_1, \dots, D_n)$ but lacks access to sanitized versions $(P_1, \dots, P_n)$.

    \item \textbf{\textup{Static  Attacker with Text Access } }: Doesn't have access to sanitized embeddings $(N_1, \dots, N_n)$ but only to the clean embeddings $(E_1, \dots, E_n).$

  \item \textbf{\textup{Adaptive Attacker with Embedding Access}} : Has access to sanitized embeddings $(N_1, \dots, N_n)$. Hence, trains a de-anonymization model to adapt to the DP mechanism $\mathcal{M}_{\textup{embedding}}$.

    \item \textbf{\textup{Adaptive  Attacker with Text Access}}: Has access to sanitized text $(P_1, \dots, P_n)$. Consequently, trains a de-anonymization model to adapt to the DP mechanism $\mathcal{M}_{\textup{text}}$.
\end{itemize}
\end{definition}

It is important to note that the adaptive attacker is a more formidable adversary since it adapts to the characteristics of the mechanism $\mathcal{M}$, whereas the static attacker only has access to clean documents/clean embeddings without any added noise.  The mode of access—either raw text or abstracted embeddings—offers further nuances, determining the exact nature of the data an attacker can exploit.




\section{DP Prompt}

Language models use a decoder network to generate sequential text output. Given a context or prompt represented by a sequence of tokens $\textup{C} = (c_1, \dots, c_{m})$, the language model generates text by sampling tokens from a conditional distribution $\textup{Pr}_{|\textup{C}}(x_1, \dots, x_{n}) = \prod_{i=1}^{n} \textup{Pr}_{|\text{C}}(x_{i}|x_1, \dots, x_{i-1})$. In this distribution, the logits $\textbf{u} \in \mathbb{R}^{\abs{\mathcal{V}}}$ are transformed using the softmax function with a temperature $T$, where $p_{ij} = \frac{\exp(\frac{u_{ij}}{T})}{\sum_{j=1}^{\abs{\mathcal{V}}} \exp(\frac{u_{ij}}{T})}$, and $\mathcal{V}$ represents the vocabulary.

This process of sequentially generating text can be regarded as a problem of selecting tokens at each step. Hence, to make the generation step differentially private, one must replace it with a differentially private version of the selection process. One commonly used and well-known differentially private mechanism is the exponential mechanism \cite{mcsherry2007mechanism}, which is defined as follows:

\begin{figure*}[t!]
  \centering
  \begin{subfigure}[b]{0.24\linewidth}
    \centering
    \includegraphics[width=\linewidth]{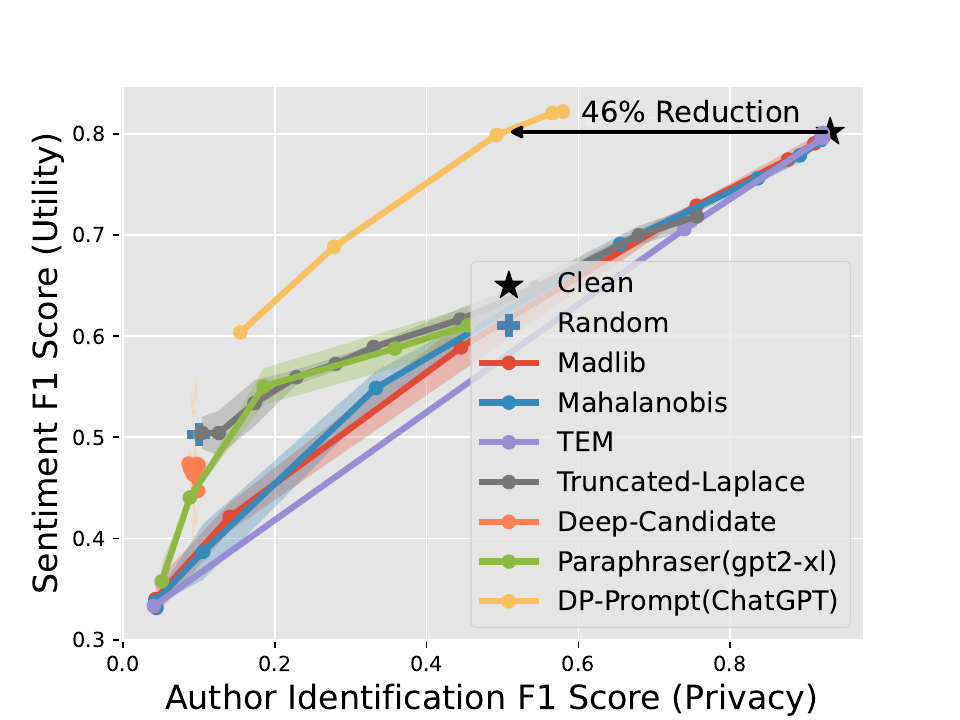}
    \subcaption{IMDB (static)}
    \label{fig:imdb_static_baseline}
  \end{subfigure}
  \begin{subfigure}[b]{0.24\linewidth}
    \centering
    \includegraphics[width=\linewidth]{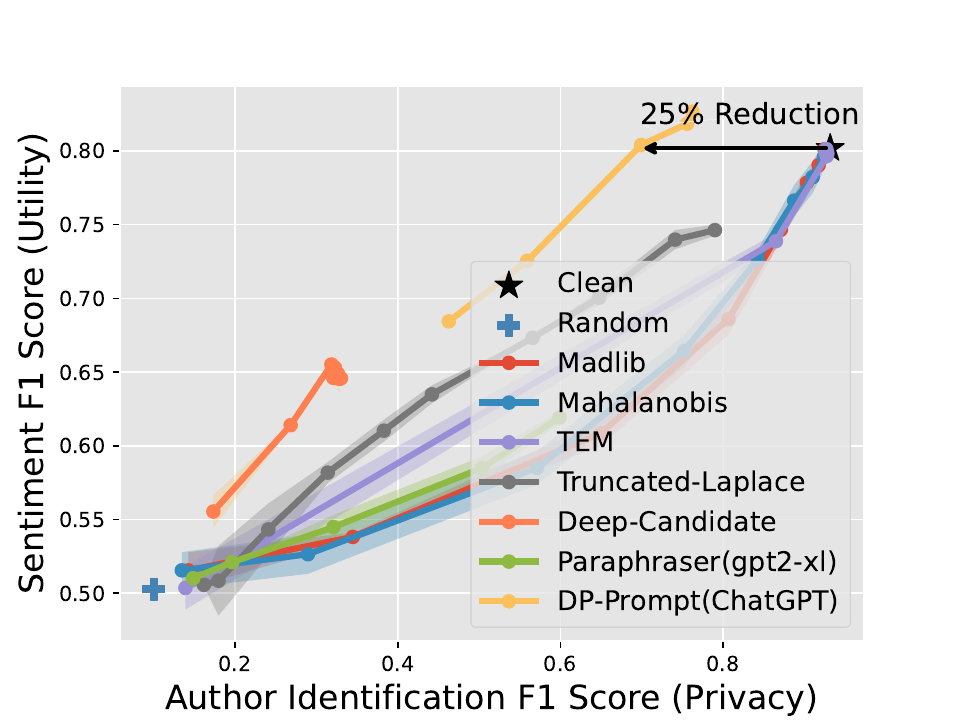}
    \subcaption{IMDB (adaptive)}
    \label{fig:imdb_adaptive_baseline}
  \end{subfigure}
  \begin{subfigure}[b]{0.24\linewidth}
    \centering
    \includegraphics[width=\linewidth]{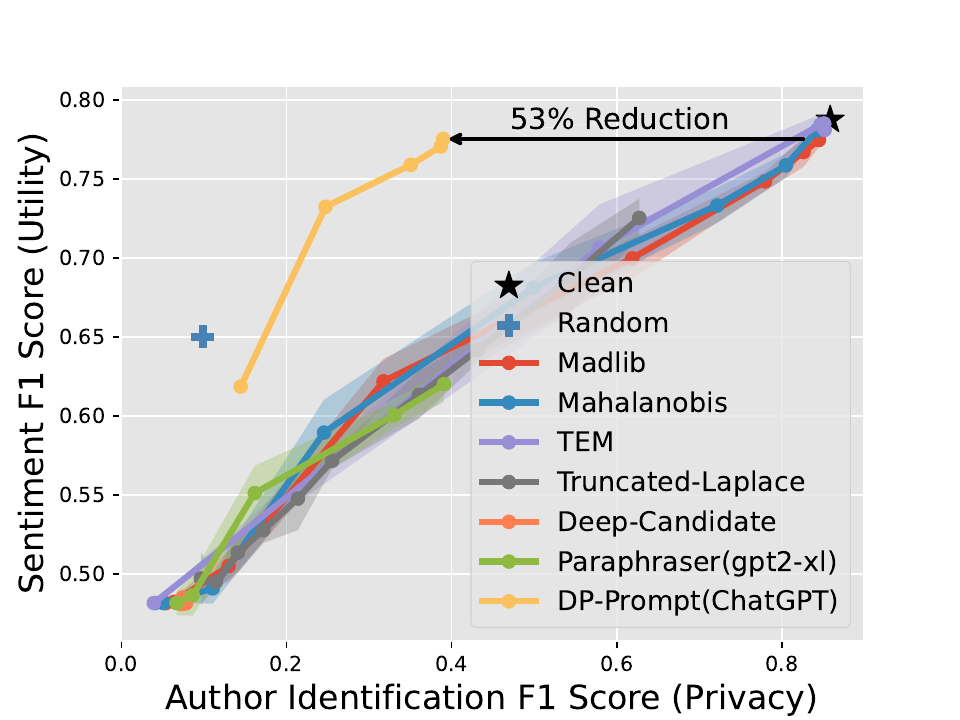}
    \subcaption{Yelp (static)}
    \label{fig:yelp_static_baseline}
  \end{subfigure}
  \begin{subfigure}[b]{0.24\linewidth}
    \centering
    \includegraphics[width=\linewidth]{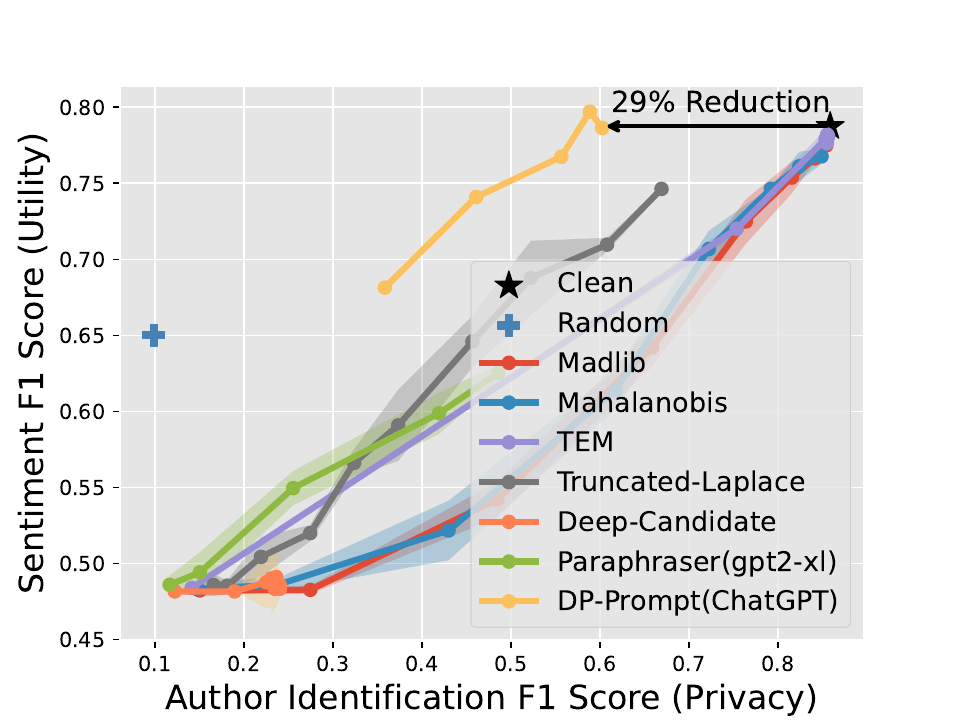}
    \subcaption{Yelp (adaptive)}
    \label{fig:yelp_adaptive_baseline}
  \end{subfigure}  
  
  \begin{subfigure}[b]{0.24\linewidth}
    \centering
    \includegraphics[width=\linewidth]{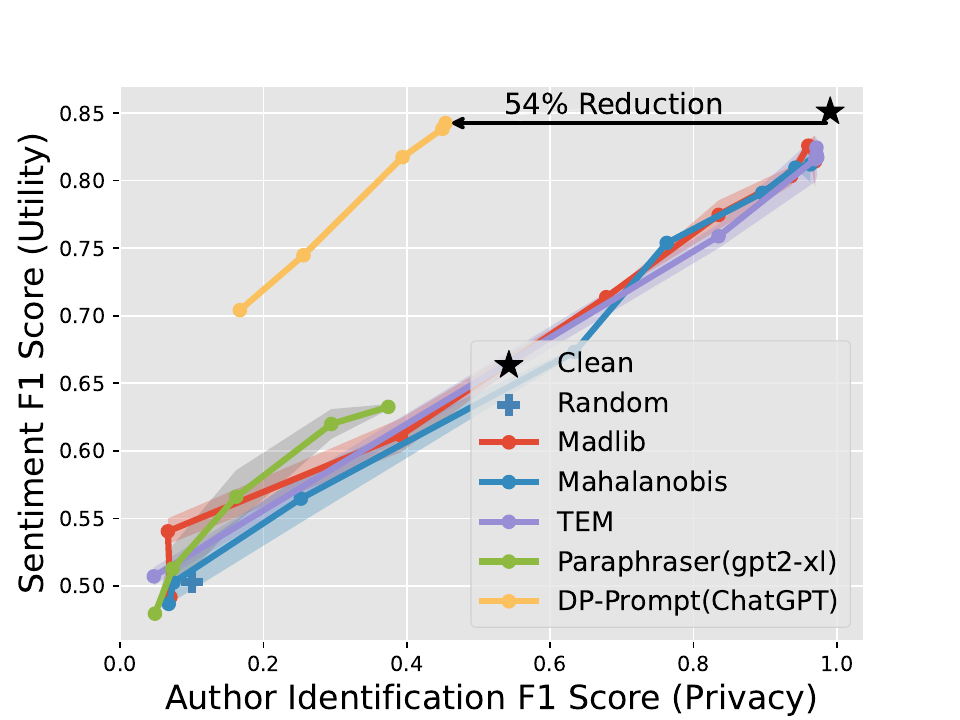}
    \subcaption{IMDB (static)}
    \label{fig:text_imdb_static_baseline}
  \end{subfigure}
  \begin{subfigure}[b]{0.24\linewidth}
    \centering
    \includegraphics[width=\linewidth]{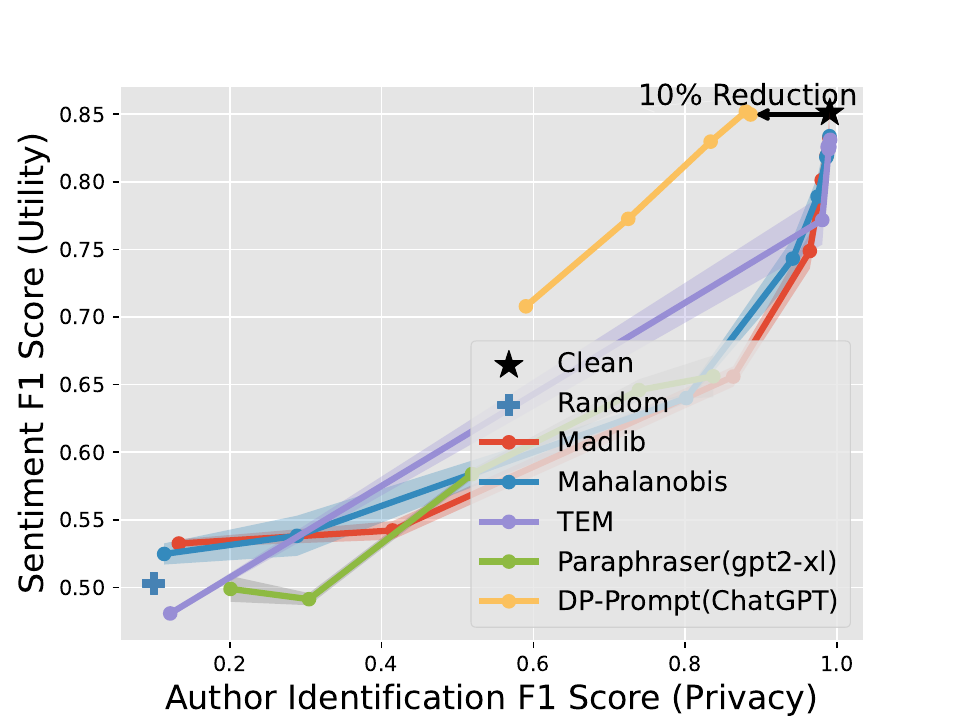}
    \subcaption{IMDB (adaptive)}
    \label{fig:text_imdb_adaptive_baseline}
  \end{subfigure}
  \begin{subfigure}[b]{0.24\linewidth}
    \centering
    \includegraphics[width=\linewidth]{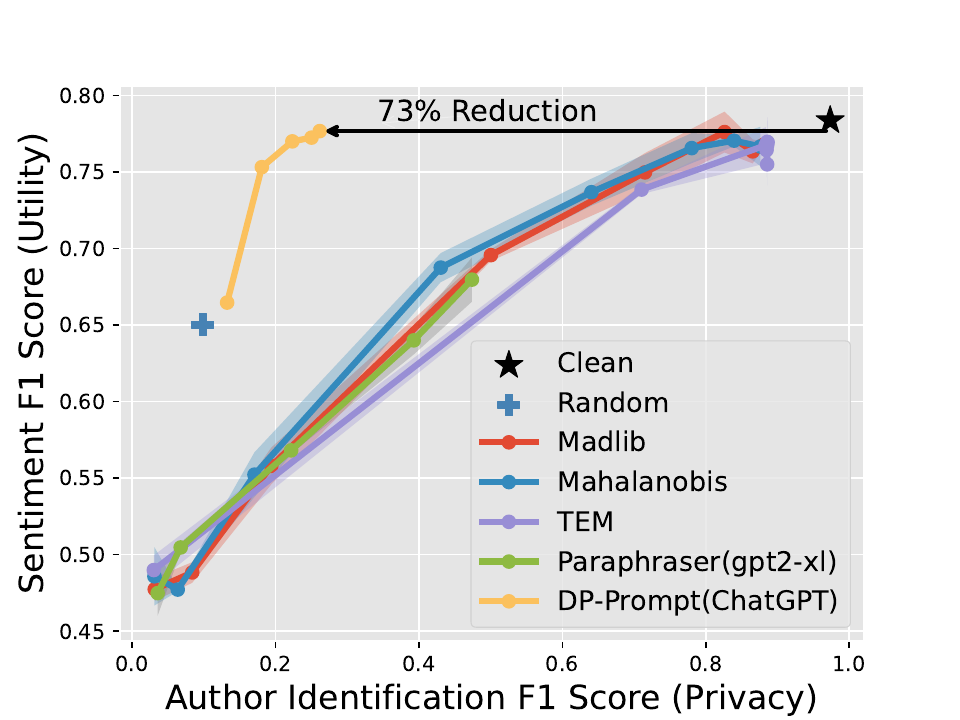}
    \subcaption{Yelp (static)}
    \label{fig:text_yelp_static_baseline}
  \end{subfigure}
  \begin{subfigure}[b]{0.24\linewidth}
    \centering
    \includegraphics[width=\linewidth]{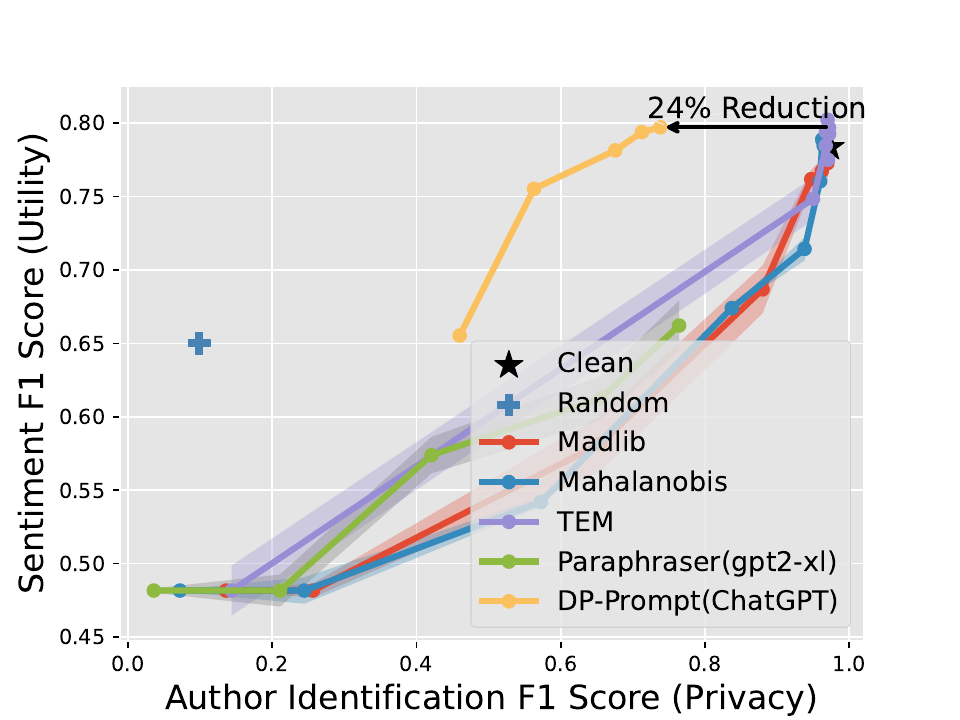}
    \subcaption{Yelp (adaptive)}
    \label{fig:text_yelp_adaptive_baseline}
  \end{subfigure}  
  \caption{Comparison of DP-Prompt (with ChatGPT) with various baselines. The top row shows results for an attacker with embedding access, while the row below presents results for an attacker with text access. Notably, it is evident that regardless of the chosen privacy level, DP-Prompt, when utilized with ChatGPT (GPT-3.5), exhibits significantly better utility compared to all baseline mechanisms.}
  \label{fig:baseline}
\end{figure*}

\begin{definition}[Exponential Mechanism] Given an utility function $\textup{u} : \mathcal{D} \times \mathcal{V} \rightarrow \mathcal{V}.$
    The exponential mechanism $\mathcal{M}_{\textup{Exp}} : \mathcal{D} \rightarrow \mathcal{V}$ is a randomized algorithm with output distribution $\textup{P}\left[ \mathcal{M}_{\textup{Exp}}(\textup{D}) = v  \right] \propto \exp{\left( \frac{\epsilon \textup{u}(\textup{D},v) }{2 \Delta u} \right)}$, where $\Delta \textup{u} = \max_{\textup{D}, \textup{D}',  v } |\textup{u}(\textup{D},v) - \textup{u}(\textup{D}', v) |$ is sensitivity.
\end{definition}

In our case, the utility of token $v_j \in \mathcal{V}$ at each step $i$ is simply the logit $u_{ij} \in \mathbb{R}$. Hence, one can make text generation differentially private using the exponential mechanism.

Extensive research has shown that paraphrasing documents helps conceal author identity \cite{rao2000can,bevendorff2019heuristic,mattern2022limits}. Considering recent advancements where tasks are formulated as prompts and language models are tasked to complete them \cite{raffel2020exploring,brown2020language,weifinetuned}, we directly prompt the language model to generate paraphrases. Therefore, given a private document $\textup{D}$ and a specific prompt template instructing the language model to generate a paraphrase, such as $\textup{T} :=$ "Paraphrase of the document:" we combine $\textup{D}$ and $\textup{T}$ to create a context $\textup{C}$. By utilizing this context, we execute the text generation procedure in a differentially private manner to produce a paraphrase. We refer to this procedure as DP-Prompt. Algorithm \ref{alg:dpprompt} outlines the specific steps of our proposed DP-Prompt, and the code for implementing DP-Prompt in HuggingFace \cite{wolf2019huggingface} is provided in Appendix \ref{sec:appendix_code}. The formal guarantee of achieving $\epsilon$-PureLDP is provided by the following theorem:
\begin{theorem}
    \label{th:LDP-Garuntee}
   Suppose the language model has not been pretrained on the private documents distribution $\mathcal{D}$. If the final logits $\textbf{u} \in \mathbb{R}^{\abs{\mathcal{V}}}$  satisfy the condition $b_1 \leq u_i \leq b_2, \forall i$, and the \textup{DP-Prompt} run  with a temperature $T$ for generating $n$ tokens, then it can be proven that the generated output satisfies $(2n(b_2-b_1)/T)$-\textup{LDP}.
\end{theorem}

See the Appendix \ref{sec:appendix_proof} for the proof.

\begin{algorithm}[t]
\small 
\SetAlgoLined
\DontPrintSemicolon
\KwIn{language model (LM), private document (D), prompt template (T),
clipping vector $\textbf{b} \in \mathbb{R}^{\lvert\mathcal{V}\rvert}$,
temperature $T \in \mathbb{R}_{+}$, paraphrase tokens n.}
\KwOut{Sanitized Doc (P)}
\BlankLine
\SetKwFunction{ClipAndScale}{ClipAndScale}
\SetKwFunction{ConvertToProbabilities}{ConvertToProbabilities}
\SetKwFunction{SampleTokens}{SampleToken}
\SetKwFunction{Tokenize}{Tokenize}
\SetKwFunction{Detokenize}{Detokenize}
\SetKwFunction{GeneratePrompt}{GeneratePrompt}

\BlankLine
$\textup{P} \leftarrow []$,~$\textup{C} \leftarrow \GeneratePrompt(\textup{D}, \textup{T})$\;
$\textup{C}_{\text{tokens}} \leftarrow \Tokenize(\textup{C)}$\;
\For{$i \leftarrow 1$ \KwTo $n$}{
  $\textbf{u} \leftarrow \text{LM}(\textup{C}_{\text{tokens}})$\;
  $\textbf{u}' \leftarrow \ClipAndScale(\textbf{u}, \textbf{b}, T)$\;
  $\textbf{p} \leftarrow \ConvertToProbabilities(\textbf{u}')$\;
  $v \leftarrow \SampleTokens(\textbf{p})$\;
  $\textup{P} \leftarrow \textup{P} \cup [v]$, $\textup{C}_{\text{tokens}} \leftarrow \textup{C}_{\text{tokens}} \cup [v]$\;
}
$\textup{P} \leftarrow \Detokenize(\textup{P})$
\Return $\textup{P}$
\caption{DP-Prompt}
\label{alg:dpprompt}
\end{algorithm}

\section{Experiments}\label{sec:experiments}


\subsection{Experiment Setup}
\label{subsec:exp_setup}



\noindent \textbf{Evaluation:}  Note that we are comparing DP-mechanisms with different levels of differential privacy. Therefore, in our experiments, we focus on evaluating the empirical privacy rather than the theoretical privacy$(\epsilon)$ for effective and realistic assessment. As a result, we plot the author identification F1 score, which is calculated by conducting de-anonymization attacks  on the sanitized documents. This score indicates the potential for privacy breaches. On the other hand, the y-axis represents the sentiment F1 score, which measures the utility of the sanitized documents.   


\begin{figure*}[htbp!]
  \centering
  \begin{subfigure}[b]{0.24\linewidth}
    \centering
    \includegraphics[width=\linewidth]{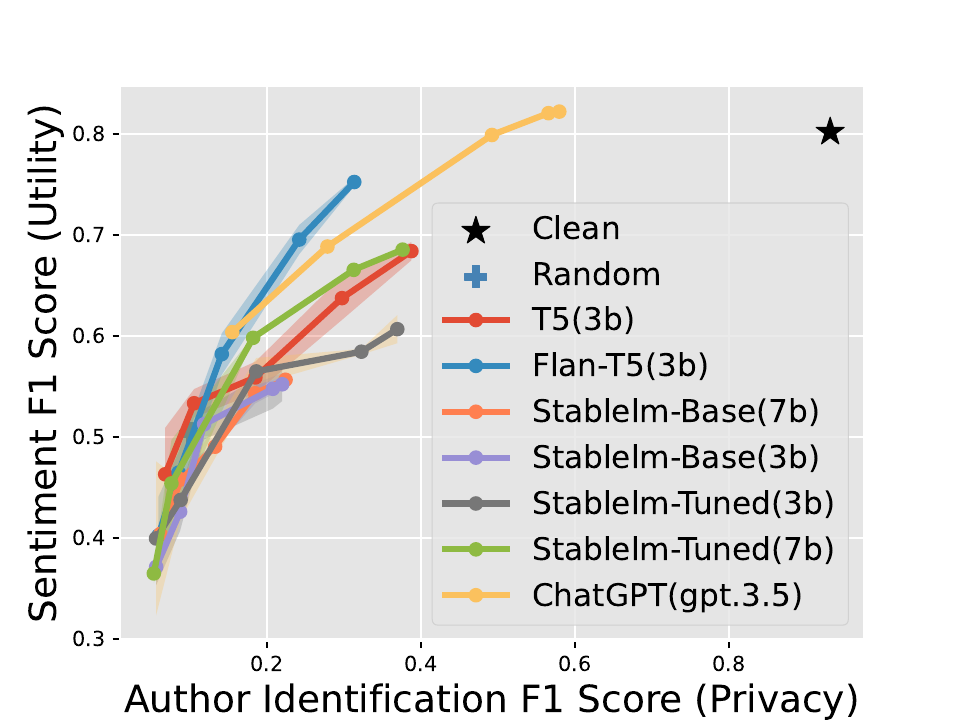}
    \subcaption{IMDB (static)}
    \label{fig:imdb_static}
  \end{subfigure}
  \begin{subfigure}[b]{0.24\linewidth}
    \centering
    \includegraphics[width=\linewidth]{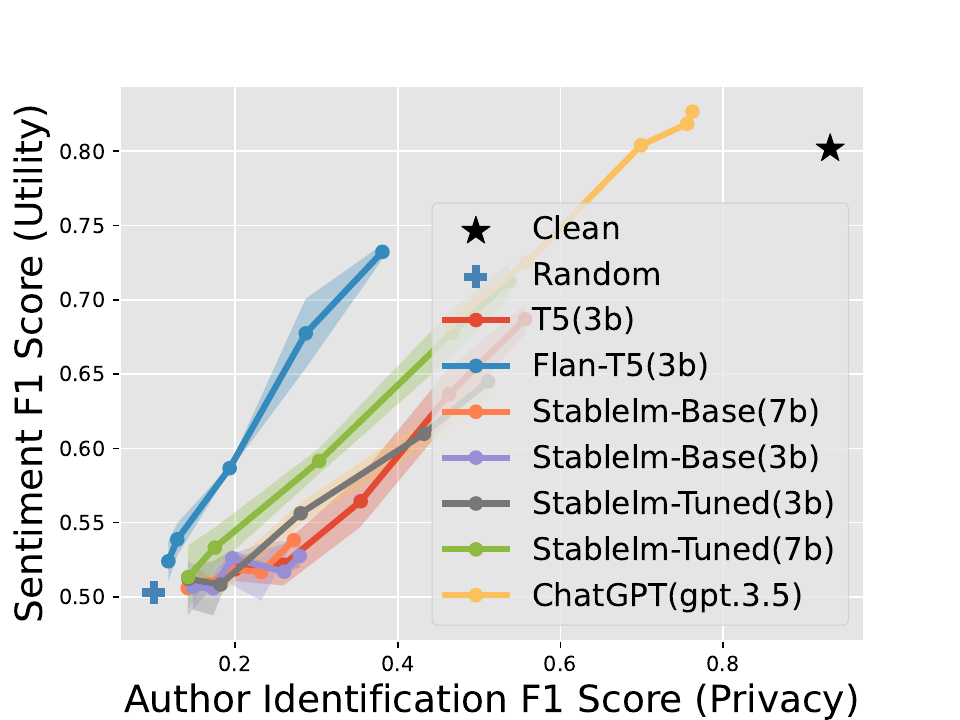}
    \subcaption{IMDB (adaptive)}
    \label{fig:imdb_adaptive}
  \end{subfigure}
  \begin{subfigure}[b]{0.24\linewidth}
    \centering
    \includegraphics[width=\linewidth]{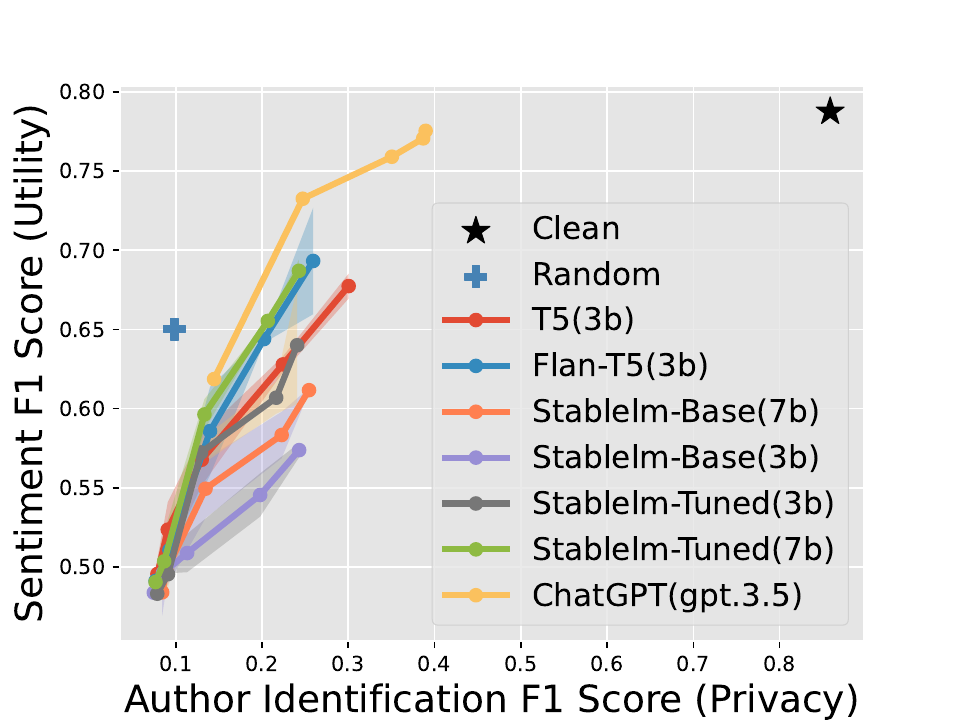}
    \subcaption{Yelp (static)}
    \label{fig:yelp_static}
  \end{subfigure}
  \begin{subfigure}[b]{0.24\linewidth}
    \centering
    \includegraphics[width=\linewidth]{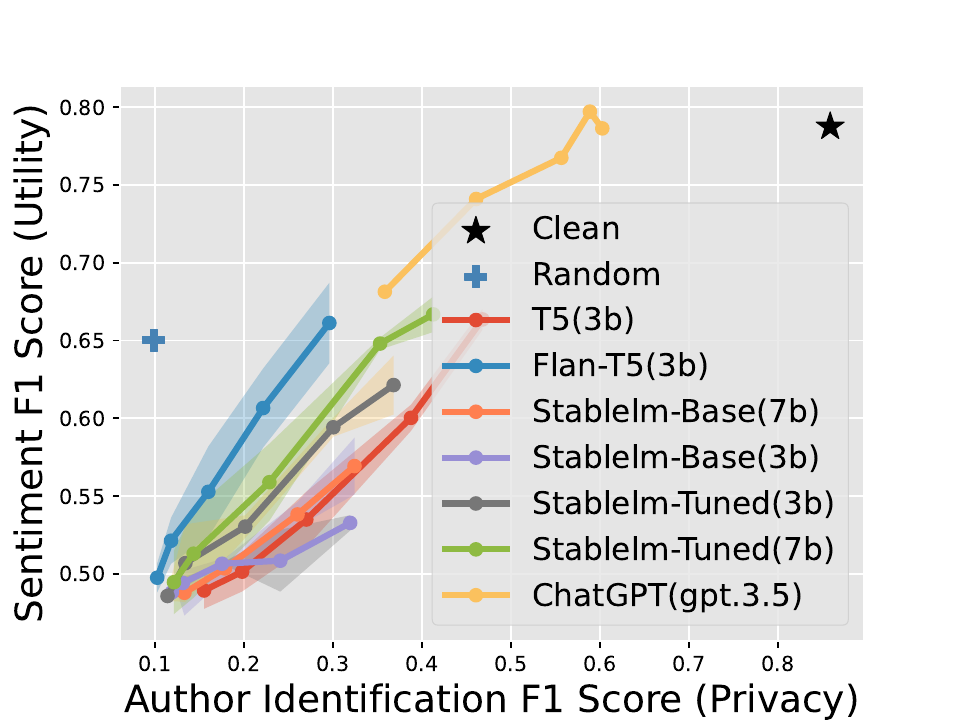}
    \subcaption{Yelp (adaptive)}
    \label{fig:yelp_adaptive}
  \end{subfigure}  

  \begin{subfigure}[b]{0.24\linewidth}
    \centering
    \includegraphics[width=\linewidth]{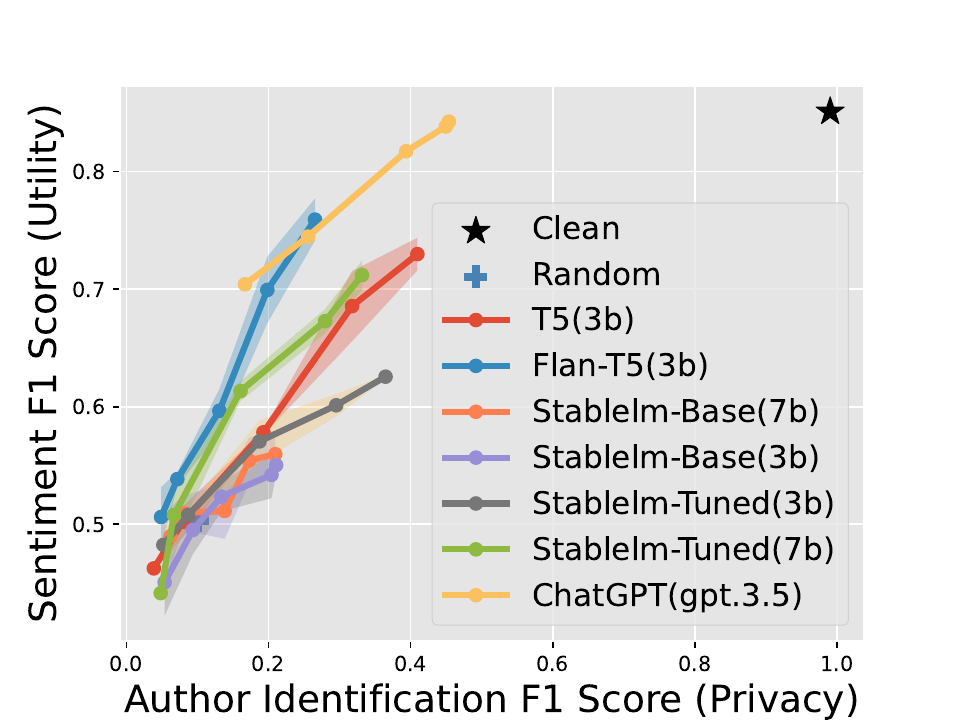}
    \subcaption{IMDB (static)}
    \label{fig:text_imdb_static}
  \end{subfigure}
  \begin{subfigure}[b]{0.24\linewidth}
    \centering
    \includegraphics[width=\linewidth]{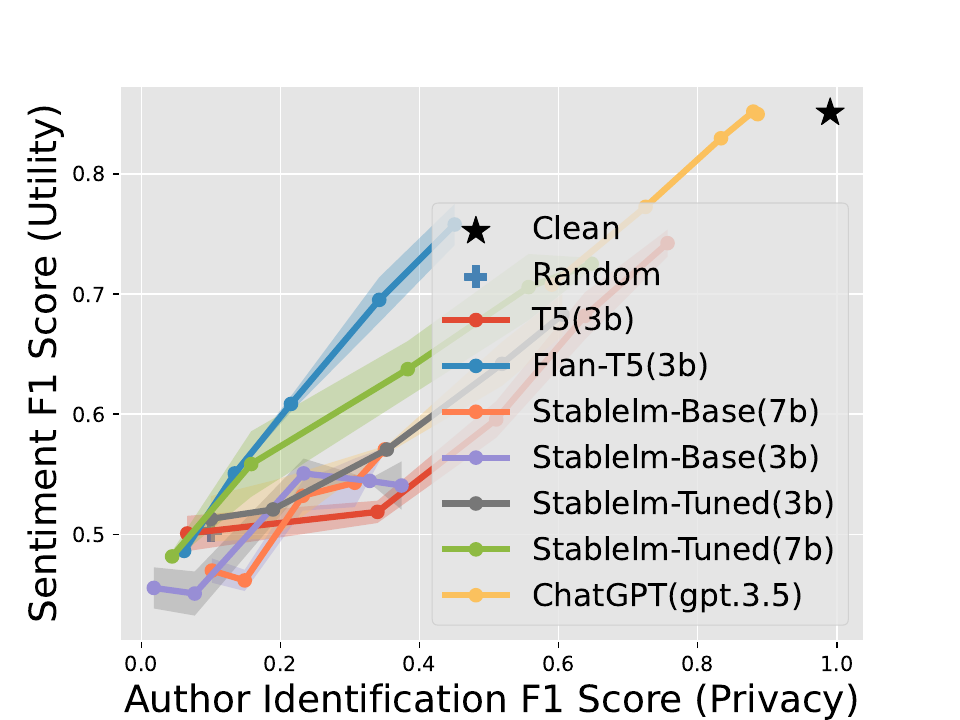}
    \subcaption{IMDB (adaptive)}
    \label{fig:text_imdb_adaptive}
  \end{subfigure}
  \begin{subfigure}[b]{0.24\linewidth}
    \centering
    \includegraphics[width=\linewidth]{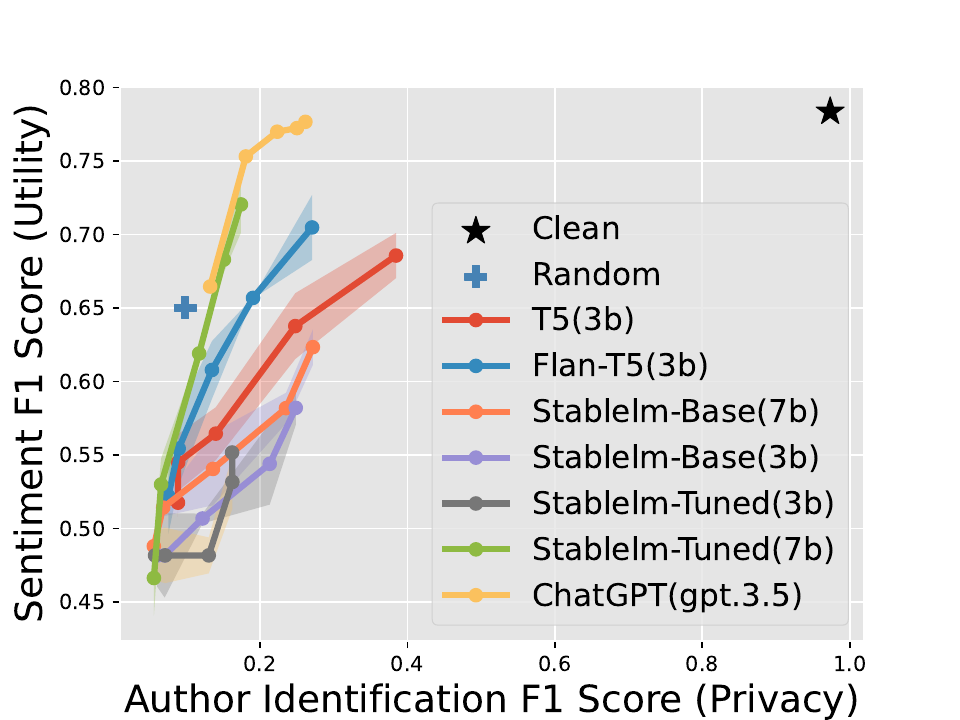}
    \subcaption{Yelp (static)}
    \label{fig:text_yelp_static}
  \end{subfigure}
  \begin{subfigure}[b]{0.24\linewidth}
    \centering
    \includegraphics[width=\linewidth]{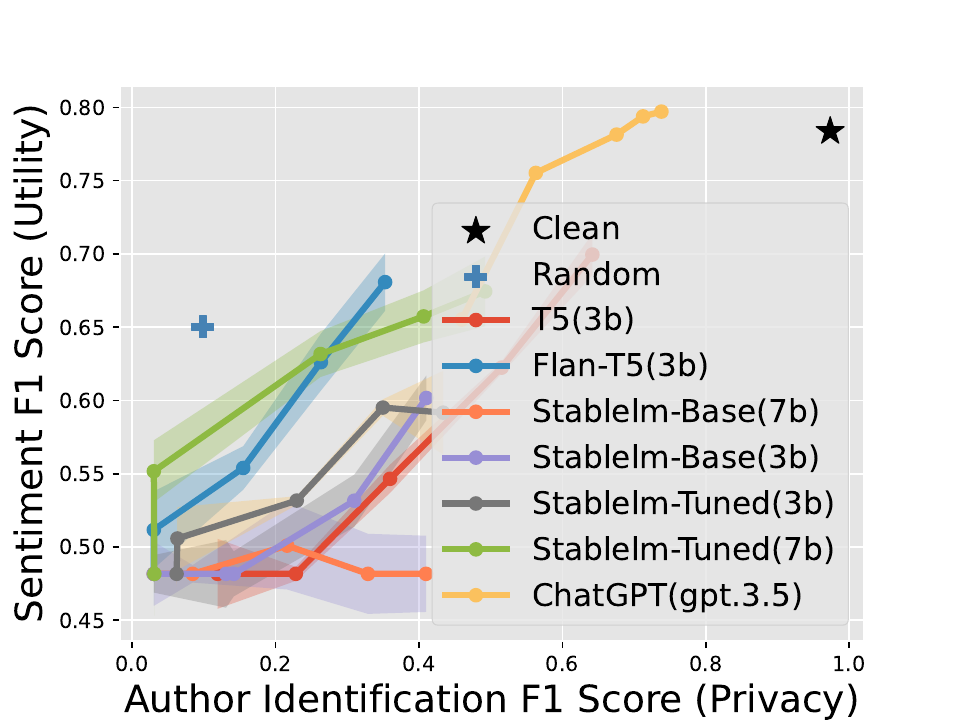}
    \subcaption{Yelp (adaptive)}
    \label{fig:text_yelp_adaptive}
  \end{subfigure}  
  \caption{
 Illustration of privacy-utility tradeoff in DP-Prompt with open source models and ChatGPT(gpt-3.5). The top row shows results for an attacker with embedding access, while the row below presents results for an attacker with text access.}
  \label{fig:open_source}
\end{figure*}

\noindent \textbf{Datasets:} We conduct experiments using IMDB movie reviews and Yelp business reviews, both of which contain author and sentiment labels. The IMDB dataset has a size of 15,000, while the Yelp dataset has 17,336 samples. For both datasets, sentiment analysis is a 2-class classification task, and the author identification task is a 10-class classification task.

\noindent \textbf{Implementation Details:} For the embedding-level attacker, we utilize 3-Layer MLPs with ReLU activation functions and train them on sentence embeddings \cite{reimers2019sentence}. For the text-level attacker, we fine-tune BERT \cite{devlin2018bert}. More details can be found in Appendix \ref{sec:appendix_imp_details}. Regarding the static attacker, the clean set of documents is used for training and validation, while the sanitized documents serve as the test set. On the other hand, for the adaptive attacker, all three sets (training, validation, and testing) consist of sanitized documents. 

\begin{itemize}
    \item For each of word level mechanisms, (Madlib \cite{feyisetan2020privacy}, Mahalanobis \cite{xu2020differentially}, TEM \cite{carvalho2021tem}) we run the mechanisms for $8$ $\epsilon$'s given $\epsilon = \{2,5,8,11,14,17,20,25\}$
    \item For each of sentence level mechanisms (Truncated-Laplace \cite{meehan2022sentence}, Deep-Candidate \cite{meehan2022sentence}): we run the mechanisms for $11$ $\epsilon$'s given by $\epsilon = \{ 5, 10, 20, 30, 40, 50, 75, 100,  150,  200\}.$ 
    \item For Paraphraser  \cite{mattern2022limits} and DP-Prompt with open source models we run decoding at $5$ temperatures $\{0.75, 1.0, 1.25, 1.5, 1.75\}.$ For DP-Prompt we run ChatGPT at temperatures $\{1.0, 1.25, 1.5, 1.75,2.0\}.$
\end{itemize}

Further we also consider F1 scores on Clean (without noise added) embeddings/documents  and performance of uniformly random classifier  (for more details, refer to Appendix \ref{subsec:random_classifier}).

\subsection{DP-prompt with ChatGPT (gpt-3.5)}\label{subsec:baseline}


In this section we compare $6$ baselines 
(Madlib, Mahalanobis, Tem, Truncated-laplace, Deep-candidate, Paraphraser) run with configurations above (for more details also refer to Appendix \ref{sec:appendix_imp_details}) with DP-Prompt with ChatGPT. Except for DP-Prompt, we run each mechanism to $3$ times to produce $3$ different sanitized documents and plot mean author F1 identification score on x-axis and show 2$\sigma$ band around mean sentiment F1 score. Results are show in Figure \ref{fig:baseline}

 The results clearly demonstrate the superior performance of DP-Prompt with ChatGPT (GPT-3.5). Notably, DP-Prompt exhibits significantly higher utility on the y-axis for a chosen empirical privacy value on the x-axis. All word-level mechanisms show a similar privacy-utility tradeoff. Regarding sentence-level mechanisms, the truncated Laplace mechanism performs decently, while in the static attack experiments, Deep-candidate is reduced to a random classifier due to the distribution shift caused by sentence recoding.

Furthermore, in the case of clean reviews (i.e., without any noise), the embedding-level attacker can accurately identify the author among 10 different options with a high F1 score of 0.93 in IMDB and 0.86 in Yelp. However, when DP-Prompt is employed, the sentiment F1 scores remain unchanged, while the author identification scores decrease by 46\% and 25\% in the case of IMDB, and 53\% and 29\% in the case of Yelp.

The text-level models are more accurate than the embedding-level models, with author identification scores of 0.99 (as opposed to 0.93) and 0.97 (as opposed to 0.86) in IMDB and Yelp, respectively, for clean reviews. When DP-Prompt is employed, the sentiment F1 scores remain unchanged, while the author identification scores decrease by 54\% and 10\% in the case of IMDB, and 73\% and 24\% in the case of Yelp. This illustrates that text-level attackers are more powerful.

\renewcommand{\arraystretch}{1.05}
\begin{table*}[t]
\centering
\resizebox{\textwidth}{!}{%
\begin{tabular}{|c|c|c|cccccccccc|cccccccccc|}
\hline
\multirow{3}{*}{} &
  \multirow{3}{*}{} &
  \textbf{Data} &
  \multicolumn{10}{c|}{\textbf{IMDB}} &
  \multicolumn{10}{c|}{\textbf{Yelp}} \\ \cline{3-23} 
 &
   &
  \textbf{Metric} &
  \multicolumn{5}{c|}{\textbf{Sentiment F1 score}} &
  \multicolumn{5}{c|}{\textbf{Author Identification F1 Score}} &
  \multicolumn{5}{c|}{\textbf{Sentiment F1 score}} &
  \multicolumn{5}{c|}{\textbf{Author Identification F1 Score}} \\ \cline{3-23} 
 &
   &
  \textbf{clipping} &
  \multicolumn{1}{c|}{\textbf{0.75}} &
  \multicolumn{1}{c|}{\textbf{1.0}} &
  \multicolumn{1}{c|}{\textbf{1.25}} &
  \multicolumn{1}{c|}{\textbf{1.5}} &
  \multicolumn{1}{c|}{\textbf{1.75}} &
  \multicolumn{1}{c|}{\textbf{0.75}} &
  \multicolumn{1}{c|}{\textbf{1.0}} &
  \multicolumn{1}{c|}{\textbf{1.25}} &
  \multicolumn{1}{c|}{\textbf{1.5}} &
  \textbf{1.75} &
  \multicolumn{1}{c|}{\textbf{0.75}} &
  \multicolumn{1}{c|}{\textbf{1.0}} &
  \multicolumn{1}{c|}{\textbf{1.25}} &
  \multicolumn{1}{c|}{\textbf{1.5}} &
  \multicolumn{1}{c|}{\textbf{1.75}} &
  \multicolumn{1}{c|}{\textbf{0.75}} &
  \multicolumn{1}{c|}{\textbf{1.0}} &
  \multicolumn{1}{c|}{\textbf{1.25}} &
  \multicolumn{1}{c|}{\textbf{1.5}} &
  \textbf{1.75} \\ \hline
\multirow{4}{*}{\textbf{\begin{tabular}[c]{@{}c@{}}Flan-t5\\ (3b)\end{tabular}}} &
  \multirow{2}{*}{\textbf{\begin{tabular}[c]{@{}c@{}}Static \\ Attacker\end{tabular}}} &
  \textbf{Yes} &
  \multicolumn{1}{c|}{0.74} &
  \multicolumn{1}{c|}{0.67} &
  \multicolumn{1}{c|}{0.56} &
  \multicolumn{1}{c|}{0.45} &
  \multicolumn{1}{c|}{0.38} &
  \multicolumn{1}{c|}{0.26} &
  \multicolumn{1}{c|}{0.21} &
  \multicolumn{1}{c|}{0.13} &
  \multicolumn{1}{c|}{0.07} &
  0.05 &
  \multicolumn{1}{c|}{0.69} &
  \multicolumn{1}{c|}{0.62} &
  \multicolumn{1}{c|}{0.58} &
  \multicolumn{1}{c|}{0.54} &
  \multicolumn{1}{c|}{0.48} &
  \multicolumn{1}{c|}{0.21} &
  \multicolumn{1}{c|}{0.17} &
  \multicolumn{1}{c|}{0.12} &
  \multicolumn{1}{c|}{0.09} &
  0.07 \\ \cline{3-23} 
 &
   &
  \textbf{No} &
  \multicolumn{1}{c|}{\begin{tabular}[c]{@{}c@{}}0.75\\ (+0.01)\end{tabular}} &
  \multicolumn{1}{c|}{\begin{tabular}[c]{@{}c@{}}0.69\\ (+0.02)\end{tabular}} &
  \multicolumn{1}{c|}{\begin{tabular}[c]{@{}c@{}}0.58\\ (+0.02)\end{tabular}} &
  \multicolumn{1}{c|}{\begin{tabular}[c]{@{}c@{}}0.46\\ (+0.01)\end{tabular}} &
  \multicolumn{1}{c|}{\begin{tabular}[c]{@{}c@{}}0.40\\ (+0.02)\end{tabular}} &
  \multicolumn{1}{c|}{\begin{tabular}[c]{@{}c@{}}0.31\\ (+0.05)\end{tabular}} &
  \multicolumn{1}{c|}{\begin{tabular}[c]{@{}c@{}}0.24\\ (+0.03)\end{tabular}} &
  \multicolumn{1}{c|}{\begin{tabular}[c]{@{}c@{}}0.14\\ (+0.01)\end{tabular}} &
  \multicolumn{1}{c|}{\begin{tabular}[c]{@{}c@{}}0.08\\ (+0.01)\end{tabular}} &
  \begin{tabular}[c]{@{}c@{}}0.05\\ (+0.00)\end{tabular} &
  \multicolumn{1}{c|}{\begin{tabular}[c]{@{}c@{}}0.69\\ (+0.00)\end{tabular}} &
  \multicolumn{1}{c|}{\begin{tabular}[c]{@{}c@{}}0.64\\ (+0.02)\end{tabular}} &
  \multicolumn{1}{c|}{\begin{tabular}[c]{@{}c@{}}0.58\\ (+0.00)\end{tabular}} &
  \multicolumn{1}{c|}{\begin{tabular}[c]{@{}c@{}}0.54\\ (+0.00)\end{tabular}} &
  \multicolumn{1}{c|}{\begin{tabular}[c]{@{}c@{}}0.49\\ (+0.01)\end{tabular}} &
  \multicolumn{1}{c|}{\begin{tabular}[c]{@{}c@{}}0.25\\ (+0.04)\end{tabular}} &
  \multicolumn{1}{c|}{\begin{tabular}[c]{@{}c@{}}0.20\\ (+0.03)\end{tabular}} &
  \multicolumn{1}{c|}{\begin{tabular}[c]{@{}c@{}}0.13\\ (+0.01)\end{tabular}} &
  \multicolumn{1}{c|}{\begin{tabular}[c]{@{}c@{}}0.09\\ (+0.00)\end{tabular}} &
  \begin{tabular}[c]{@{}c@{}}0.07\\ (+0.00)\end{tabular} \\ \cline{2-23} 
 &
  \multirow{2}{*}{\textbf{\begin{tabular}[c]{@{}c@{}}Adaptive\\ Attacker\end{tabular}}} &
  \textbf{Yes} &
  \multicolumn{1}{c|}{0.72} &
  \multicolumn{1}{c|}{0.66} &
  \multicolumn{1}{c|}{0.58} &
  \multicolumn{1}{c|}{0.52} &
  \multicolumn{1}{c|}{0.51} &
  \multicolumn{1}{c|}{0.34} &
  \multicolumn{1}{c|}{0.25} &
  \multicolumn{1}{c|}{0.16} &
  \multicolumn{1}{c|}{0.13} &
  0.11 &
  \multicolumn{1}{c|}{0.67} &
  \multicolumn{1}{c|}{0.59} &
  \multicolumn{1}{c|}{0.54} &
  \multicolumn{1}{c|}{0.52} &
  \multicolumn{1}{c|}{0.49} &
  \multicolumn{1}{c|}{0.25} &
  \multicolumn{1}{c|}{0.20} &
  \multicolumn{1}{c|}{0.14} &
  \multicolumn{1}{c|}{0.11} &
  0.10 \\ \cline{3-23} 
 &
   &
  \textbf{No} &
  \multicolumn{1}{c|}{\begin{tabular}[c]{@{}c@{}}0.73\\ (+0.01)\end{tabular}} &
  \multicolumn{1}{c|}{\begin{tabular}[c]{@{}c@{}}0.67\\ (+0.01)\end{tabular}} &
  \multicolumn{1}{c|}{\begin{tabular}[c]{@{}c@{}}0.58\\ (+0.00)\end{tabular}} &
  \multicolumn{1}{c|}{\begin{tabular}[c]{@{}c@{}}0.53\\ (+0.01)\end{tabular}} &
  \multicolumn{1}{c|}{\begin{tabular}[c]{@{}c@{}}0.52\\ (+0.01)\end{tabular}} &
  \multicolumn{1}{c|}{\begin{tabular}[c]{@{}c@{}}0.38\\ (+0.04)\end{tabular}} &
  \multicolumn{1}{c|}{\begin{tabular}[c]{@{}c@{}}0.28\\ (+0.03)\end{tabular}} &
  \multicolumn{1}{c|}{\begin{tabular}[c]{@{}c@{}}0.19\\ (+0.03)\end{tabular}} &
  \multicolumn{1}{c|}{\begin{tabular}[c]{@{}c@{}}0.13\\ (+0.00)\end{tabular}} &
  \begin{tabular}[c]{@{}c@{}}0.11\\ (+0.00)\end{tabular} &
  \multicolumn{1}{c|}{\begin{tabular}[c]{@{}c@{}}0.67\\ (0.00)\end{tabular}} &
  \multicolumn{1}{c|}{\begin{tabular}[c]{@{}c@{}}0.60\\ (+0.01)\end{tabular}} &
  \multicolumn{1}{c|}{\begin{tabular}[c]{@{}c@{}}0.55\\ (+0.01)\end{tabular}} &
  \multicolumn{1}{c|}{\begin{tabular}[c]{@{}c@{}}0.52\\ (+0.00)\end{tabular}} &
  \multicolumn{1}{c|}{\begin{tabular}[c]{@{}c@{}}0.49\\ +0.00)\end{tabular}} &
  \multicolumn{1}{c|}{\begin{tabular}[c]{@{}c@{}}0.29\\ (+0.04)\end{tabular}} &
  \multicolumn{1}{c|}{\begin{tabular}[c]{@{}c@{}}0.22\\ (+0.02)\end{tabular}} &
  \multicolumn{1}{c|}{\begin{tabular}[c]{@{}c@{}}0.15\\ (+0.01)\end{tabular}} &
  \multicolumn{1}{c|}{\begin{tabular}[c]{@{}c@{}}0.11\\ (+0.01)\end{tabular}} &
  0.10 \\ \hline
\multirow{4}{*}{\textbf{\begin{tabular}[c]{@{}c@{}}Stablelm\\ Tuned\\ (7b)\end{tabular}}} &
  \multirow{2}{*}{\textbf{\begin{tabular}[c]{@{}c@{}}Static\\ Attacker\end{tabular}}} &
  \textbf{Yes} &
  \multicolumn{1}{c|}{0.67} &
  \multicolumn{1}{c|}{0.63} &
  \multicolumn{1}{c|}{0.53} &
  \multicolumn{1}{c|}{0.38} &
  \multicolumn{1}{c|}{0.34} &
  \multicolumn{1}{c|}{0.33} &
  \multicolumn{1}{c|}{0.29} &
  \multicolumn{1}{c|}{0.12} &
  \multicolumn{1}{c|}{0.05} &
  0.03 &
  \multicolumn{1}{c|}{0.66} &
  \multicolumn{1}{c|}{0.62} &
  \multicolumn{1}{c|}{0.55} &
  \multicolumn{1}{c|}{0.49} &
  \multicolumn{1}{c|}{0.48} &
  \multicolumn{1}{c|}{0.28} &
  \multicolumn{1}{c|}{0.22} &
  \multicolumn{1}{c|}{0.10} &
  \multicolumn{1}{c|}{0.08} &
  0.07 \\ \cline{3-23} 
 &
   &
  \textbf{No} &
  \multicolumn{1}{c|}{\begin{tabular}[c]{@{}c@{}}0.68\\ (+0.01)\end{tabular}} &
  \multicolumn{1}{c|}{\begin{tabular}[c]{@{}c@{}}0.66\\ (+0.03)\end{tabular}} &
  \multicolumn{1}{c|}{\begin{tabular}[c]{@{}c@{}}0.59\\ (+0.06)\end{tabular}} &
  \multicolumn{1}{c|}{\begin{tabular}[c]{@{}c@{}}0.45\\ (+0.07)\end{tabular}} &
  \multicolumn{1}{c|}{\begin{tabular}[c]{@{}c@{}}0.36\\ (+0.02)\end{tabular}} &
  \multicolumn{1}{c|}{\begin{tabular}[c]{@{}c@{}}0.37\\ (+0.04)\end{tabular}} &
  \multicolumn{1}{c|}{\begin{tabular}[c]{@{}c@{}}0.31\\ (+0.02)\end{tabular}} &
  \multicolumn{1}{c|}{\begin{tabular}[c]{@{}c@{}}0.18\\ (+0.06)\end{tabular}} &
  \multicolumn{1}{c|}{\begin{tabular}[c]{@{}c@{}}0.07\\ (+0.02)\end{tabular}} &
  \begin{tabular}[c]{@{}c@{}}0.05\\ (+0.02)\end{tabular} &
  \multicolumn{1}{c|}{\begin{tabular}[c]{@{}c@{}}0.68\\ (+0.02)\end{tabular}} &
  \multicolumn{1}{c|}{\begin{tabular}[c]{@{}c@{}}0.65\\ (+0.03)\end{tabular}} &
  \multicolumn{1}{c|}{\begin{tabular}[c]{@{}c@{}}0.59\\ (+0.04)\end{tabular}} &
  \multicolumn{1}{c|}{\begin{tabular}[c]{@{}c@{}}0.50\\ (+0.01)\end{tabular}} &
  \multicolumn{1}{c|}{\begin{tabular}[c]{@{}c@{}}0.49\\ (+0.01)\end{tabular}} &
  \multicolumn{1}{c|}{\begin{tabular}[c]{@{}c@{}}0.26\\ (-0.02)\end{tabular}} &
  \multicolumn{1}{c|}{\begin{tabular}[c]{@{}c@{}}0.21\\ (-0.01)\end{tabular}} &
  \multicolumn{1}{c|}{\begin{tabular}[c]{@{}c@{}}0.13\\ (+0.03)\end{tabular}} &
  \multicolumn{1}{c|}{\begin{tabular}[c]{@{}c@{}}0.08\\ (+0.00)\end{tabular}} &
  \begin{tabular}[c]{@{}c@{}}0.07\\ (+0.00)\end{tabular} \\ \cline{2-23} 
 &
  \multirow{2}{*}{\textbf{\begin{tabular}[c]{@{}c@{}}Adaptive\\ Attacker\end{tabular}}} &
  \textbf{Yes} &
  \multicolumn{1}{c|}{0.65} &
  \multicolumn{1}{c|}{0.60} &
  \multicolumn{1}{c|}{0.54} &
  \multicolumn{1}{c|}{0.51} &
  \multicolumn{1}{c|}{0.50} &
  \multicolumn{1}{c|}{0.45} &
  \multicolumn{1}{c|}{0.37} &
  \multicolumn{1}{c|}{0.22} &
  \multicolumn{1}{c|}{0.14} &
  0.12 &
  \multicolumn{1}{c|}{0.63} &
  \multicolumn{1}{c|}{0.57} &
  \multicolumn{1}{c|}{0.50} &
  \multicolumn{1}{c|}{0.49} &
  \multicolumn{1}{c|}{0.48} &
  \multicolumn{1}{c|}{0.36} &
  \multicolumn{1}{c|}{0.28} &
  \multicolumn{1}{c|}{0.17} &
  \multicolumn{1}{c|}{0.13} &
  0.10 \\ \cline{3-23} 
 &
   &
  \textbf{No} &
  \multicolumn{1}{c|}{\begin{tabular}[c]{@{}c@{}}0.71\\ (+0.06)\end{tabular}} &
  \multicolumn{1}{c|}{\begin{tabular}[c]{@{}c@{}}0.67\\ (+0.07)\end{tabular}} &
  \multicolumn{1}{c|}{\begin{tabular}[c]{@{}c@{}}0.59\\ (+0.05)\end{tabular}} &
  \multicolumn{1}{c|}{\begin{tabular}[c]{@{}c@{}}0.53\\ (+0.02)\end{tabular}} &
  \multicolumn{1}{c|}{\begin{tabular}[c]{@{}c@{}}0.51\\ (0.01)\end{tabular}} &
  \multicolumn{1}{c|}{\begin{tabular}[c]{@{}c@{}}0.53\\ (+0.08)\end{tabular}} &
  \multicolumn{1}{c|}{\begin{tabular}[c]{@{}c@{}}0.46\\ (+0.09)\end{tabular}} &
  \multicolumn{1}{c|}{\begin{tabular}[c]{@{}c@{}}0.30\\ (+0.02)\end{tabular}} &
  \multicolumn{1}{c|}{\begin{tabular}[c]{@{}c@{}}0.17\\ (+0.03)\end{tabular}} &
  \begin{tabular}[c]{@{}c@{}}0.14\\ (+0.02)\end{tabular} &
  \multicolumn{1}{c|}{\begin{tabular}[c]{@{}c@{}}0.66\\ (+0.03)\end{tabular}} &
  \multicolumn{1}{c|}{\begin{tabular}[c]{@{}c@{}}0.64\\ (+0.07)\end{tabular}} &
  \multicolumn{1}{c|}{\begin{tabular}[c]{@{}c@{}}0.55\\ (+0.05)\end{tabular}} &
  \multicolumn{1}{c|}{\begin{tabular}[c]{@{}c@{}}0.51\\ (+0.02)\end{tabular}} &
  \multicolumn{1}{c|}{\begin{tabular}[c]{@{}c@{}}0.49\\ (0.01)\end{tabular}} &
  \multicolumn{1}{c|}{\begin{tabular}[c]{@{}c@{}}0.41\\ (+0.05)\end{tabular}} &
  \multicolumn{1}{c|}{\begin{tabular}[c]{@{}c@{}}0.35\\ (+0.07)\end{tabular}} &
  \multicolumn{1}{c|}{\begin{tabular}[c]{@{}c@{}}0.22\\ (+0.05)\end{tabular}} &
  \multicolumn{1}{c|}{\begin{tabular}[c]{@{}c@{}}0.14\\ (+0.01)\end{tabular}} &
  \begin{tabular}[c]{@{}c@{}}0.12\\ (+0.02)\end{tabular} \\ \hline
\end{tabular}%
}
\caption{Effect of clipping (shown as Yes) and without clipping (shown as No) on privacy-utility tradeoff.}
\label{tab:clipping}
\end{table*}
\subsection{DP-Prompt with open source models} \label{subsec:open-source}

It is important to note that models such as ChatGPT (gpt-3.5) are proprietary and can only be accessed through APIs, necessitating the uploading of user documents to the language model provider. Although DP-Prompt with such proprietary models provide LDP guarantee, this defeats the fundamental motivation of LDP, which is to achieve privacy guarantees without relying on a trusted data curator. The  objective of the experiments presented in the preceding section aim to demonstrate that DP-Prompt, when combined with a powerful language model like ChatGPT, can outperform existing methods by a significant margin.

Considering the increasing interest in building high-quality open-source large language models \citep{scao2022bloom,  black2022gpt,  touvron2023llama, li2023textbooks, jiang2023mistral}, we expand our evaluation of DP-Prompt to include six open-source models, ranging in size up to seven billion parameters. Our evaluation takes into account two factors: (i) architecture, where we consider both encoder-decoder and decoder-only models, and (ii) the level of fine-tuning, including models that are fine-tuned using instructions and/or Reinforcement Learning with Human Feedback (RLHF). Specifically $6$ models are as follows,  \textit{Base}:  T5 (3b),  Stable lm base (3b, 7b),\textit{Instruction finetuned/RLHF tuned:} Flan T5 (3b), Stable lm tuned (3b,7b).

While T5, Flan T5 are encoder-decoder models, rest of the models are  decoder-only. In contrast to DP-Prompt with ChatGPT, we perform DP-Prompt using the aforementioned open source language models three times for each dataset and temperature, resulting in three sanitized documents. Further we all run these open source models at half precision $(\text{torch.float16})$ and with max number of tokens in paraphrase to $150$. The obtained results are presented in Figure \ref{fig:open_source}.
 
Now we compare open source models along various factors 

\noindent \textbf{Base vs Instruction/RLHF tuned:} It is evident that the base models perform poorly in comparison to the models that underwent instruction fine-tuning/RLHF tuning. One important point to mention is that both StableLM-Based (3b, 7b) perform significantly worse against adaptive attackers.

\noindent \textbf{Scale:} We observe that scale plays a key role, as StableLM-tuned (7b) demonstrates better utility than StableLM-Tuned (3b) across all levels of empirical privacy, with ChatGPT outperforming both.

\noindent \textbf{Variance:} It is worth noting that we did not obtain variance for DP-Prompt with ChatGPT in Figure \ref{fig:baseline} through multiple runs. However, our analysis in Figure \ref{fig:open_source} suggests that there is no significant variance in the Sentiment F1 score, at least for open source models.

\noindent  \textbf{Encoder-Decoder / Decoder only:} Flan-T5(3b), an encoder-decoder model, outperforms the larger Stablelm-Tuned (7b) model. Notably, Flan-T5(3b) is fine-tuned exclusively on academic datasets, resulting in shorter paraphrases compared to Stablelm-Tuned (7b).

\noindent \textbf{ChatGPT vs Rest:} While the open-source models we considered demonstrate competitiveness with ChatGPT, a notable gap remains. The key finding is that even at a higher temperature of 1.5, ChatGPT is capable of recovering a clean sentiment F1 score, while none of the open-source models can achieve a matching clean sentiment F1 score, even at a significantly lower temperature of 0.75. See Figure \ref{fig:chat_gpt_vs_stablelm_tuned}, which presents the paraphrases output by ChatGPT and Stablelm-Tuned (7B).

%


\renewcommand{\arraystretch}{1.05}
\begin{table*}[t]
\centering
\resizebox{\textwidth}{!}{%
\begin{tabular}{|c|c|c|cccccccccc|cccccccccc|}
\hline
\multirow{3}{*}{} &
  \multirow{3}{*}{} &
  \textbf{Data} &
  \multicolumn{10}{c|}{\textbf{IMDB}} &
  \multicolumn{10}{c|}{\textbf{Yelp}} \\ \cline{3-23} 
 &
   &
  \textbf{Metric} &
  \multicolumn{5}{c|}{\textbf{Sentiment F1 score}} &
  \multicolumn{5}{c|}{\textbf{Author Identification F1 Score}} &
  \multicolumn{5}{c|}{\textbf{Sentiment F1 score}} &
  \multicolumn{5}{c|}{\textbf{Author Identification F1 Score}} \\ \cline{3-23} 
 &
   &
  \textbf{top\_k} &
  \multicolumn{1}{c|}{\textbf{0.75}} &
  \multicolumn{1}{c|}{\textbf{1.0}} &
  \multicolumn{1}{c|}{\textbf{1.25}} &
  \multicolumn{1}{c|}{\textbf{1.5}} &
  \multicolumn{1}{c|}{\textbf{1.75}} &
  \multicolumn{1}{c|}{\textbf{0.75}} &
  \multicolumn{1}{c|}{\textbf{1.0}} &
  \multicolumn{1}{c|}{\textbf{1.25}} &
  \multicolumn{1}{c|}{\textbf{1.5}} &
  \textbf{1.75} &
  \multicolumn{1}{c|}{\textbf{0.75}} &
  \multicolumn{1}{c|}{\textbf{1.0}} &
  \multicolumn{1}{c|}{\textbf{1.25}} &
  \multicolumn{1}{c|}{\textbf{1.5}} &
  \multicolumn{1}{c|}{\textbf{1.75}} &
  \multicolumn{1}{c|}{\textbf{0.75}} &
  \multicolumn{1}{c|}{\textbf{1.0}} &
  \multicolumn{1}{c|}{\textbf{1.25}} &
  \multicolumn{1}{c|}{\textbf{1.5}} &
  \textbf{1.75} \\ \hline
\multirow{6}{*}{\textbf{Flan-t5}} &
  \multirow{3}{*}{\textbf{\begin{tabular}[c]{@{}c@{}}Static \\ Attacker\end{tabular}}} &
  \textbf{all} &
  \multicolumn{1}{c|}{0.75} &
  \multicolumn{1}{c|}{0.69} &
  \multicolumn{1}{c|}{0.58} &
  \multicolumn{1}{c|}{0.46} &
  \multicolumn{1}{c|}{0.40} &
  \multicolumn{1}{c|}{0.31} &
  \multicolumn{1}{c|}{0.24} &
  \multicolumn{1}{c|}{0.14} &
  \multicolumn{1}{c|}{0.08} &
  0.05 &
  \multicolumn{1}{c|}{0.69} &
  \multicolumn{1}{c|}{0.64} &
  \multicolumn{1}{c|}{0.58} &
  \multicolumn{1}{c|}{0.51} &
  \multicolumn{1}{c|}{0.49} &
  \multicolumn{1}{c|}{0.25} &
  \multicolumn{1}{c|}{0.20} &
  \multicolumn{1}{c|}{0.13} &
  \multicolumn{1}{c|}{0.09} &
  0.07 \\ \cline{3-23} 
 &
   &
  \textbf{80} &
  \multicolumn{1}{c|}{\begin{tabular}[c]{@{}c@{}}0.75\\ (+0.0)\end{tabular}} &
  \multicolumn{1}{c|}{\begin{tabular}[c]{@{}c@{}}0.72\\ (+0.03)\end{tabular}} &
  \multicolumn{1}{c|}{\begin{tabular}[c]{@{}c@{}}0.70\\ (+0.12)\end{tabular}} &
  \multicolumn{1}{c|}{\begin{tabular}[c]{@{}c@{}}0.69\\ (+0.23)\end{tabular}} &
  \multicolumn{1}{c|}{\begin{tabular}[c]{@{}c@{}}0.66\\ (+0.26)\end{tabular}} &
  \multicolumn{1}{c|}{\begin{tabular}[c]{@{}c@{}}0.31\\ (+0.0)\end{tabular}} &
  \multicolumn{1}{c|}{\begin{tabular}[c]{@{}c@{}}0.27\\ (+0.03)\end{tabular}} &
  \multicolumn{1}{c|}{\begin{tabular}[c]{@{}c@{}}0.23\\ (+0.09)\end{tabular}} &
  \multicolumn{1}{c|}{\begin{tabular}[c]{@{}c@{}}0.22\\ (+0.14)\end{tabular}} &
  \begin{tabular}[c]{@{}c@{}}0.21\\ (+0.16)\end{tabular} &
  \multicolumn{1}{c|}{\begin{tabular}[c]{@{}c@{}}0.71\\ (+0.02)\end{tabular}} &
  \multicolumn{1}{c|}{\begin{tabular}[c]{@{}c@{}}0.69\\ (+0.05)\end{tabular}} &
  \multicolumn{1}{c|}{\begin{tabular}[c]{@{}c@{}}0.68\\ (+0.10)\end{tabular}} &
  \multicolumn{1}{c|}{\begin{tabular}[c]{@{}c@{}}0.64\\ (+0.13)\end{tabular}} &
  \multicolumn{1}{c|}{\begin{tabular}[c]{@{}c@{}}0.64\\ (+0.15)\end{tabular}} &
  \multicolumn{1}{c|}{\begin{tabular}[c]{@{}c@{}}0.25\\ (+0.00)\end{tabular}} &
  \multicolumn{1}{c|}{\begin{tabular}[c]{@{}c@{}}0.20\\ (+0.00)\end{tabular}} &
  \multicolumn{1}{c|}{\begin{tabular}[c]{@{}c@{}}0.19\\ (+0.06)\end{tabular}} &
  \multicolumn{1}{c|}{\begin{tabular}[c]{@{}c@{}}0.19\\ (+0.10)\end{tabular}} &
  \begin{tabular}[c]{@{}c@{}}0.17\\ (+0.10)\end{tabular} \\ \cline{3-23} 
 &
   &
  \textbf{40} &
  \multicolumn{1}{c|}{\begin{tabular}[c]{@{}c@{}}0.75\\ (+0.00)\end{tabular}} &
  \multicolumn{1}{c|}{\begin{tabular}[c]{@{}c@{}}0.74\\ (+0.05)\end{tabular}} &
  \multicolumn{1}{c|}{\begin{tabular}[c]{@{}c@{}}0.72\\ (+0.24)\end{tabular}} &
  \multicolumn{1}{c|}{\begin{tabular}[c]{@{}c@{}}0.69\\ (+0.23)\end{tabular}} &
  \multicolumn{1}{c|}{\begin{tabular}[c]{@{}c@{}}0.67\\ (+0.27)\end{tabular}} &
  \multicolumn{1}{c|}{\begin{tabular}[c]{@{}c@{}}0.32\\ (+0.0)\end{tabular}} &
  \multicolumn{1}{c|}{\begin{tabular}[c]{@{}c@{}}0.29\\ (+0.05)\end{tabular}} &
  \multicolumn{1}{c|}{\begin{tabular}[c]{@{}c@{}}0.26\\ (+0.12)\end{tabular}} &
  \multicolumn{1}{c|}{\begin{tabular}[c]{@{}c@{}}0.25\\ (+0.17)\end{tabular}} &
  \begin{tabular}[c]{@{}c@{}}0.23\\ (+0.18)\end{tabular} &
  \multicolumn{1}{c|}{\begin{tabular}[c]{@{}c@{}}0.71\\ (+0.02)\end{tabular}} &
  \multicolumn{1}{c|}{\begin{tabular}[c]{@{}c@{}}0.69\\ (+0.05)\end{tabular}} &
  \multicolumn{1}{c|}{\begin{tabular}[c]{@{}c@{}}0.68\\ (+0.10)\end{tabular}} &
  \multicolumn{1}{c|}{\begin{tabular}[c]{@{}c@{}}0.64\\ (+0.13)\end{tabular}} &
  \multicolumn{1}{c|}{\begin{tabular}[c]{@{}c@{}}0.64\\ (+0.15)\end{tabular}} &
  \multicolumn{1}{c|}{\begin{tabular}[c]{@{}c@{}}0.25\\ (+0.00)\end{tabular}} &
  \multicolumn{1}{c|}{\begin{tabular}[c]{@{}c@{}}0.22\\ (+0.02)\end{tabular}} &
  \multicolumn{1}{c|}{\begin{tabular}[c]{@{}c@{}}0.20\\ (+0.07)\end{tabular}} &
  \multicolumn{1}{c|}{\begin{tabular}[c]{@{}c@{}}0.19\\ (+0.10)\end{tabular}} &
  \begin{tabular}[c]{@{}c@{}}0.17\\ (+0.10)\end{tabular} \\ \cline{2-23} 
 &
  \multirow{3}{*}{\textbf{\begin{tabular}[c]{@{}c@{}}Adaptive\\ Attacker\end{tabular}}} &
  \textbf{all} &
  \multicolumn{1}{c|}{0.73} &
  \multicolumn{1}{c|}{0.67} &
  \multicolumn{1}{c|}{0.58} &
  \multicolumn{1}{c|}{0.53} &
  \multicolumn{1}{c|}{0.52} &
  \multicolumn{1}{c|}{0.38} &
  \multicolumn{1}{c|}{0.28} &
  \multicolumn{1}{c|}{0.19} &
  \multicolumn{1}{c|}{0.12} &
  0.11 &
  \multicolumn{1}{c|}{0.66} &
  \multicolumn{1}{c|}{0.60} &
  \multicolumn{1}{c|}{0.55} &
  \multicolumn{1}{c|}{0.52} &
  \multicolumn{1}{c|}{0.49} &
  \multicolumn{1}{c|}{0.29} &
  \multicolumn{1}{c|}{0.22} &
  \multicolumn{1}{c|}{0.15} &
  \multicolumn{1}{c|}{0.11} &
  0.10 \\ \cline{3-23} 
 &
   &
  \textbf{80} &
  \multicolumn{1}{c|}{\begin{tabular}[c]{@{}c@{}}0.76\\ (+0.03)\end{tabular}} &
  \multicolumn{1}{c|}{\begin{tabular}[c]{@{}c@{}}0.70\\ (+0.03)\end{tabular}} &
  \multicolumn{1}{c|}{\begin{tabular}[c]{@{}c@{}}0.69\\ (+0.11)\end{tabular}} &
  \multicolumn{1}{c|}{\begin{tabular}[c]{@{}c@{}}0.66\\ (+0.13)\end{tabular}} &
  \multicolumn{1}{c|}{\begin{tabular}[c]{@{}c@{}}0.65\\ (+0.13)\end{tabular}} &
  \multicolumn{1}{c|}{\begin{tabular}[c]{@{}c@{}}0.39\\ (+0.01)\end{tabular}} &
  \multicolumn{1}{c|}{\begin{tabular}[c]{@{}c@{}}0.33\\ (+0.05)\end{tabular}} &
  \multicolumn{1}{c|}{\begin{tabular}[c]{@{}c@{}}0.29\\ (+0.10)\end{tabular}} &
  \multicolumn{1}{c|}{\begin{tabular}[c]{@{}c@{}}0.27\\ (+0.15)\end{tabular}} &
  \begin{tabular}[c]{@{}c@{}}0.26\\ (+0.15)\end{tabular} &
  \multicolumn{1}{c|}{\begin{tabular}[c]{@{}c@{}}0.67\\ (+0.01)\end{tabular}} &
  \multicolumn{1}{c|}{\begin{tabular}[c]{@{}c@{}}0.66\\ (+0.06)\end{tabular}} &
  \multicolumn{1}{c|}{\begin{tabular}[c]{@{}c@{}}0.62\\ (+0.07)\end{tabular}} &
  \multicolumn{1}{c|}{\begin{tabular}[c]{@{}c@{}}0.60\\ (+0.08)\end{tabular}} &
  \multicolumn{1}{c|}{\begin{tabular}[c]{@{}c@{}}0.60\\ (+0.11)\end{tabular}} &
  \multicolumn{1}{c|}{\begin{tabular}[c]{@{}c@{}}0.31\\ (+0.02)\end{tabular}} &
  \multicolumn{1}{c|}{\begin{tabular}[c]{@{}c@{}}0.24\\ (+0.02)\end{tabular}} &
  \multicolumn{1}{c|}{\begin{tabular}[c]{@{}c@{}}0.23\\ (+0.08)\end{tabular}} &
  \multicolumn{1}{c|}{\begin{tabular}[c]{@{}c@{}}0.20\\ (+0.09)\end{tabular}} &
  \begin{tabular}[c]{@{}c@{}}0.20\\ (+0.10)\end{tabular} \\ \cline{3-23} 
 &
   &
  \textbf{40} &
  \multicolumn{1}{c|}{\begin{tabular}[c]{@{}c@{}}0.75\\ (+0.02)\end{tabular}} &
  \multicolumn{1}{c|}{\begin{tabular}[c]{@{}c@{}}0.72\\ (+0.05)\end{tabular}} &
  \multicolumn{1}{c|}{\begin{tabular}[c]{@{}c@{}}0.70\\ (+0.12)\end{tabular}} &
  \multicolumn{1}{c|}{\begin{tabular}[c]{@{}c@{}}0.67\\ (+0.14)\end{tabular}} &
  \multicolumn{1}{c|}{\begin{tabular}[c]{@{}c@{}}0.67\\ (+0.15)\end{tabular}} &
  \multicolumn{1}{c|}{\begin{tabular}[c]{@{}c@{}}0.41\\ (+0.02)\end{tabular}} &
  \multicolumn{1}{c|}{\begin{tabular}[c]{@{}c@{}}0.36\\ (+0.08)\end{tabular}} &
  \multicolumn{1}{c|}{\begin{tabular}[c]{@{}c@{}}0.33\\ (+0.14)\end{tabular}} &
  \multicolumn{1}{c|}{\begin{tabular}[c]{@{}c@{}}0.31\\ (+0.19)\end{tabular}} &
  \begin{tabular}[c]{@{}c@{}}0.29\\ (+0.18)\end{tabular} &
  \multicolumn{1}{c|}{\begin{tabular}[c]{@{}c@{}}0.69\\ (+0.02)\end{tabular}} &
  \multicolumn{1}{c|}{\begin{tabular}[c]{@{}c@{}}0.66\\ (+0.06)\end{tabular}} &
  \multicolumn{1}{c|}{\begin{tabular}[c]{@{}c@{}}0.65\\ (+0.10)\end{tabular}} &
  \multicolumn{1}{c|}{\begin{tabular}[c]{@{}c@{}}0.63\\ (+0.11)\end{tabular}} &
  \multicolumn{1}{c|}{\begin{tabular}[c]{@{}c@{}}0.61\\ (+0.12)\end{tabular}} &
  \multicolumn{1}{c|}{\begin{tabular}[c]{@{}c@{}}0.31\\ (+0.02)\end{tabular}} &
  \multicolumn{1}{c|}{\begin{tabular}[c]{@{}c@{}}0.26\\ (+0.04)\end{tabular}} &
  \multicolumn{1}{c|}{\begin{tabular}[c]{@{}c@{}}0.24\\ (+0.09)\end{tabular}} &
  \multicolumn{1}{c|}{\begin{tabular}[c]{@{}c@{}}0.23\\ (+0.12)\end{tabular}} &
  \begin{tabular}[c]{@{}c@{}}0.23\\ (+0.13)\end{tabular} \\ \hline
\multirow{6}{*}{\textbf{\begin{tabular}[c]{@{}c@{}}Stablelm\\ Tuned\\   (7b)\end{tabular}}} &
  \multirow{3}{*}{\textbf{\begin{tabular}[c]{@{}c@{}}Static\\ Attacker\end{tabular}}} &
  \textbf{all} &
  \multicolumn{1}{c|}{0.68} &
  \multicolumn{1}{c|}{0.66} &
  \multicolumn{1}{c|}{0.59} &
  \multicolumn{1}{c|}{0.45} &
  \multicolumn{1}{c|}{0.36} &
  \multicolumn{1}{c|}{0.37} &
  \multicolumn{1}{c|}{0.31} &
  \multicolumn{1}{c|}{0.18} &
  \multicolumn{1}{c|}{0.07} &
  0.05 &
  \multicolumn{1}{c|}{0.68} &
  \multicolumn{1}{c|}{0.65} &
  \multicolumn{1}{c|}{0.59} &
  \multicolumn{1}{c|}{0.50} &
  \multicolumn{1}{c|}{0.49} &
  \multicolumn{1}{c|}{0.24} &
  \multicolumn{1}{c|}{0.20} &
  \multicolumn{1}{c|}{0.13} &
  \multicolumn{1}{c|}{0.08} &
  0.07 \\ \cline{3-23} 
 &
   &
  \textbf{80} &
  \multicolumn{1}{c|}{\begin{tabular}[c]{@{}c@{}}0.69\\ (+0.01)\end{tabular}} &
  \multicolumn{1}{c|}{\begin{tabular}[c]{@{}c@{}}0.67\\ (+0.01)\end{tabular}} &
  \multicolumn{1}{c|}{\begin{tabular}[c]{@{}c@{}}0.66\\ (+0.07)\end{tabular}} &
  \multicolumn{1}{c|}{\begin{tabular}[c]{@{}c@{}}0.62\\ (+0.17)\end{tabular}} &
  \multicolumn{1}{c|}{\begin{tabular}[c]{@{}c@{}}0.62\\ (+0.26)\end{tabular}} &
  \multicolumn{1}{c|}{\begin{tabular}[c]{@{}c@{}}0.37\\ (+0.00)\end{tabular}} &
  \multicolumn{1}{c|}{\begin{tabular}[c]{@{}c@{}}0.34\\ (+0.03)\end{tabular}} &
  \multicolumn{1}{c|}{\begin{tabular}[c]{@{}c@{}}0.28\\ (+0.10)\end{tabular}} &
  \multicolumn{1}{c|}{\begin{tabular}[c]{@{}c@{}}0.27\\ (+0.20)\end{tabular}} &
  \begin{tabular}[c]{@{}c@{}}0.24\\ (+0.l9)\end{tabular} &
  \multicolumn{1}{c|}{\begin{tabular}[c]{@{}c@{}}0.69\\ (+0.01)\end{tabular}} &
  \multicolumn{1}{c|}{\begin{tabular}[c]{@{}c@{}}0.66\\ (+0.01)\end{tabular}} &
  \multicolumn{1}{c|}{\begin{tabular}[c]{@{}c@{}}0.65\\ (+0.06)\end{tabular}} &
  \multicolumn{1}{c|}{\begin{tabular}[c]{@{}c@{}}0.62\\ (+0.12)\end{tabular}} &
  \multicolumn{1}{c|}{\begin{tabular}[c]{@{}c@{}}0.59\\ (+0.10)\end{tabular}} &
  \multicolumn{1}{c|}{\begin{tabular}[c]{@{}c@{}}0.24\\ (+0.00)\end{tabular}} &
  \multicolumn{1}{c|}{\begin{tabular}[c]{@{}c@{}}0.22\\ (+0.02)\end{tabular}} &
  \multicolumn{1}{c|}{\begin{tabular}[c]{@{}c@{}}0.19\\ (+0.06)\end{tabular}} &
  \multicolumn{1}{c|}{\begin{tabular}[c]{@{}c@{}}0.16\\ (+0.08)\end{tabular}} &
  \begin{tabular}[c]{@{}c@{}}0.16\\ (+0.09)\end{tabular} \\ \cline{3-23} 
 &
   &
  \textbf{40} &
  \multicolumn{1}{c|}{\begin{tabular}[c]{@{}c@{}}0.69\\ (+0.01)\end{tabular}} &
  \multicolumn{1}{c|}{\begin{tabular}[c]{@{}c@{}}0.67\\ (+0.01)\end{tabular}} &
  \multicolumn{1}{c|}{\begin{tabular}[c]{@{}c@{}}0.66\\ (+0.07)\end{tabular}} &
  \multicolumn{1}{c|}{\begin{tabular}[c]{@{}c@{}}0.64\\ (+0.19)\end{tabular}} &
  \multicolumn{1}{c|}{\begin{tabular}[c]{@{}c@{}}0.64\\ (+0.28)\end{tabular}} &
  \multicolumn{1}{c|}{\begin{tabular}[c]{@{}c@{}}0.38\\ (+0.01)\end{tabular}} &
  \multicolumn{1}{c|}{\begin{tabular}[c]{@{}c@{}}0.34\\ (+0.03)\end{tabular}} &
  \multicolumn{1}{c|}{\begin{tabular}[c]{@{}c@{}}0.30\\ (+0.12)\end{tabular}} &
  \multicolumn{1}{c|}{\begin{tabular}[c]{@{}c@{}}0.27\\ (+0.20)\end{tabular}} &
  \begin{tabular}[c]{@{}c@{}}0.27\\ (+0.22)\end{tabular} &
  \multicolumn{1}{c|}{\begin{tabular}[c]{@{}c@{}}0.69\\ (+0.01)\end{tabular}} &
  \multicolumn{1}{c|}{\begin{tabular}[c]{@{}c@{}}0.69\\ (+0.04)\end{tabular}} &
  \multicolumn{1}{c|}{\begin{tabular}[c]{@{}c@{}}0.61\\ (+0.02)\end{tabular}} &
  \multicolumn{1}{c|}{\begin{tabular}[c]{@{}c@{}}0.65\\ (+0.15)\end{tabular}} &
  \multicolumn{1}{c|}{\begin{tabular}[c]{@{}c@{}}0.65\\ (+0.16)\end{tabular}} &
  \multicolumn{1}{c|}{\begin{tabular}[c]{@{}c@{}}0.25\\ (+0.01)\end{tabular}} &
  \multicolumn{1}{c|}{\begin{tabular}[c]{@{}c@{}}0.22\\ (+0.02)\end{tabular}} &
  \multicolumn{1}{c|}{\begin{tabular}[c]{@{}c@{}}0.20\\ (+0.07)\end{tabular}} &
  \multicolumn{1}{c|}{\begin{tabular}[c]{@{}c@{}}0.19\\ (+0.11)\end{tabular}} &
  \begin{tabular}[c]{@{}c@{}}0.18\\ (+0.11)\end{tabular} \\ \cline{2-23} 
 &
  \multirow{3}{*}{\textbf{\begin{tabular}[c]{@{}c@{}}Adaptive\\ Attacker\end{tabular}}} &
  \textbf{all} &
  \multicolumn{1}{c|}{0.70} &
  \multicolumn{1}{c|}{0.67} &
  \multicolumn{1}{c|}{0.59} &
  \multicolumn{1}{c|}{0.53} &
  \multicolumn{1}{c|}{0.51} &
  \multicolumn{1}{c|}{0.53} &
  \multicolumn{1}{c|}{0.46} &
  \multicolumn{1}{c|}{0.30} &
  \multicolumn{1}{c|}{0.17} &
  0.14 &
  \multicolumn{1}{c|}{0.66} &
  \multicolumn{1}{c|}{0.64} &
  \multicolumn{1}{c|}{0.55} &
  \multicolumn{1}{c|}{0.51} &
  \multicolumn{1}{c|}{0.49} &
  \multicolumn{1}{c|}{0.41} &
  \multicolumn{1}{c|}{0.35} &
  \multicolumn{1}{c|}{0.22} &
  \multicolumn{1}{c|}{0.14} &
  0.12 \\ \cline{3-23} 
 &
   &
  \textbf{80} &
  \multicolumn{1}{c|}{\begin{tabular}[c]{@{}c@{}}0.70\\ (+0.00)\end{tabular}} &
  \multicolumn{1}{c|}{\begin{tabular}[c]{@{}c@{}}0.68\\ (+0.01)\end{tabular}} &
  \multicolumn{1}{c|}{\begin{tabular}[c]{@{}c@{}}0.67\\ (+0.08)\end{tabular}} &
  \multicolumn{1}{c|}{\begin{tabular}[c]{@{}c@{}}0.64\\ (+0.11)\end{tabular}} &
  \multicolumn{1}{c|}{\begin{tabular}[c]{@{}c@{}}0.63\\ (+0.12)\end{tabular}} &
  \multicolumn{1}{c|}{\begin{tabular}[c]{@{}c@{}}0.53\\ (+0.00)\end{tabular}} &
  \multicolumn{1}{c|}{\begin{tabular}[c]{@{}c@{}}0.49\\ (+0.03)\end{tabular}} &
  \multicolumn{1}{c|}{\begin{tabular}[c]{@{}c@{}}0.44\\ (+0.14)\end{tabular}} &
  \multicolumn{1}{c|}{\begin{tabular}[c]{@{}c@{}}0.40\\ (+0.23)\end{tabular}} &
  \begin{tabular}[c]{@{}c@{}}0.35\\ (+0.21)\end{tabular} &
  \multicolumn{1}{c|}{\begin{tabular}[c]{@{}c@{}}0.69\\ (+0.03)\end{tabular}} &
  \multicolumn{1}{c|}{\begin{tabular}[c]{@{}c@{}}0.66\\ (+0.02)\end{tabular}} &
  \multicolumn{1}{c|}{\begin{tabular}[c]{@{}c@{}}0.64\\ (+0.09)\end{tabular}} &
  \multicolumn{1}{c|}{\begin{tabular}[c]{@{}c@{}}0.59\\ (+0.08)\end{tabular}} &
  \multicolumn{1}{c|}{\begin{tabular}[c]{@{}c@{}}0.60\\ (+0.11)\end{tabular}} &
  \multicolumn{1}{c|}{\begin{tabular}[c]{@{}c@{}}0.41\\ (+0.00)\end{tabular}} &
  \multicolumn{1}{c|}{\begin{tabular}[c]{@{}c@{}}0.38\\ (+0.03)\end{tabular}} &
  \multicolumn{1}{c|}{\begin{tabular}[c]{@{}c@{}}0.32\\ (+0.10)\end{tabular}} &
  \multicolumn{1}{c|}{\begin{tabular}[c]{@{}c@{}}0.29\\ (+0.15)\end{tabular}} &
  \begin{tabular}[c]{@{}c@{}}0.27\\ (+0.15)\end{tabular} \\ \cline{3-23} 
 &
   &
  \textbf{40} &
  \multicolumn{1}{c|}{\begin{tabular}[c]{@{}c@{}}0.70\\ (+0.00)\end{tabular}} &
  \multicolumn{1}{c|}{\begin{tabular}[c]{@{}c@{}}0.69\\ (+0.02)\end{tabular}} &
  \multicolumn{1}{c|}{\begin{tabular}[c]{@{}c@{}}0.67\\ (+0.08)\end{tabular}} &
  \multicolumn{1}{c|}{\begin{tabular}[c]{@{}c@{}}0.64\\ (+0.11)\end{tabular}} &
  \multicolumn{1}{c|}{\begin{tabular}[c]{@{}c@{}}0.63\\ (+0.12)\end{tabular}} &
  \multicolumn{1}{c|}{\begin{tabular}[c]{@{}c@{}}0.53\\ (+0.00)\end{tabular}} &
  \multicolumn{1}{c|}{\begin{tabular}[c]{@{}c@{}}0.49\\ (+0.03)\end{tabular}} &
  \multicolumn{1}{c|}{\begin{tabular}[c]{@{}c@{}}0.45\\ (+0.15)\end{tabular}} &
  \multicolumn{1}{c|}{\begin{tabular}[c]{@{}c@{}}0.42\\ (+0.25)\end{tabular}} &
  \begin{tabular}[c]{@{}c@{}}0.39\\ (+0.25)\end{tabular} &
  \multicolumn{1}{c|}{\begin{tabular}[c]{@{}c@{}}0.69\\ (+0.03)\end{tabular}} &
  \multicolumn{1}{c|}{\begin{tabular}[c]{@{}c@{}}0.67\\ (+0.03)\end{tabular}} &
  \multicolumn{1}{c|}{\begin{tabular}[c]{@{}c@{}}0.64\\ (+0.09)\end{tabular}} &
  \multicolumn{1}{c|}{\begin{tabular}[c]{@{}c@{}}0.60\\ (+0.09)\end{tabular}} &
  \multicolumn{1}{c|}{\begin{tabular}[c]{@{}c@{}}0.62\\ (+0.13)\end{tabular}} &
  \multicolumn{1}{c|}{\begin{tabular}[c]{@{}c@{}}0.42\\ (+0.01)\end{tabular}} &
  \multicolumn{1}{c|}{\begin{tabular}[c]{@{}c@{}}0.37\\ (+0.02)\end{tabular}} &
  \multicolumn{1}{c|}{\begin{tabular}[c]{@{}c@{}}0.34\\ (+0.12)\end{tabular}} &
  \multicolumn{1}{c|}{\begin{tabular}[c]{@{}c@{}}0.32\\ (+0.18)\end{tabular}} &
  \begin{tabular}[c]{@{}c@{}}0.29\\ (+0.17)\end{tabular} \\ \hline
\end{tabular}%
}
\caption{The effect of top-k sampling on privacy-utility tradeoff on Flan-T5(3b) and Stablelm-Tuned(7B) models}
\label{tab:top_k_sampling}
\end{table*}

\subsection{Effect of clipping logits}
\label{subsec:clipping}

In both sections  (Section \ref{subsec:baseline} and Section \ref{subsec:open-source}) DP-Prompt is run without logit clipping. This is primarily because ChatGPT doesn't expose logits, and for a fair comparison with ChatGPT, we didn't clip logits for open-source models in Section \ref{subsec:open-source}. However, the LDP guarantee still holds because, in practice, logits are always bounded within a certain precision, even though they may have large values.

In this section, we examine the impact of logit clipping on the tradeoff between privacy and utility using FlanT5(3b) and Stablelm Tuned(7b) models against embedding-level attacks. We adopt the approach of \cite{li2021differentially} by learning the clipping boundaries on additional data (more details can be found in Appendix \ref{sec:appendix_imp_details}). The results of this analysis are presented in Table \ref{tab:clipping}. The maximum difference observed with and without clipping occurs for Stablelm-Tuned(7b) on the IMDB dataset at a temperature of $0.75$. In this case, the sentiment F1 score drops from $0.71$ to $0.65$, and the author identification F1 score drops from $0.53$ to $0.45$. Based on these findings, we recommend using logit clipping when higher privacy is required and not using clipping when higher utility is desired.

\subsection{Effect of Top-K sampling}\label{subsec:top-k}

For the differential privacy guarantee to hold, sampling from probabilities must be done over the entire vocabulary according to score probabilities. While it is common to use top-k sampling in practice to improve generation (the default value in Hugging Face \cite{wolf2019huggingface} is 40). It is important to note that the ChatGPT chat completion API does not include a top-k parameter. In this section, using open-source models, we examine the impact of top-k sampling on utility and investigate whether it provides any empirical privacy, even in the \emph{absence} of the differential privacy guarantee. In addition, we aim to assess whether top-k sampling can effectively help open source models to recover the clean sentiment F1-score and narrow the gap compared to DP-Prompt when used with ChatGPT.

We consider Flan-T5(3b) and Stablelm-Tuned(7B) models and perform decoding with top-k sampling for $k=\{40, 80\}$. The results against embedding level attacks, including no top-k sampling, are shown in Table \ref{tab:top_k_sampling}. The maximum difference observed with and without clipping occurs for Stablelm-Tuned(0.75) on the IMDB dataset at a temperature of $0.7$ for $k=40$. In this case, the sentiment F1 score increases from $0.32$ to $0.64$, and the author identification F1 score increases from $0.05$ to $0.27$. Additionally, top-k sampling does not improve utility at a temperature of $0.75$, indicating that DP-Prompt with open-source models, even with top-k sampling, does not match DP-Prompt with ChatGPT.

\section{Issue of Data Memorization}
It is important to acknowledge the possibility that ChatGPT and other open-source language models (LLMs) may have been exposed to online review datasets such as IMDB and Yelp. This exposure is not unique to our approach. Similar questions can be raised for word-level approaches \cite{feyisetan2020privacy, xu2020differentially, carvalho2021tem} and sentence-level approaches \cite{meehan2022sentence} since GloVe embeddings \cite{pennington2014glove} and Sentence Transformers (BERT) \cite{reimers2019sentence} are pretrained on the Common Crawl dataset. Hence assumption in Theorem \ref{th:LDP-Garuntee} may not hold. We argue that this exposure does not pose a major concern for DP-Prompt due to the following reasons. 

\noindent\textbf{No availability of paraphrases of IMDB, Yelp:}  DP-Prompt essentially evaluates the performance of LLMs in the task of zero-shot paraphrasing at higher temperatures. Although the language models may have been fine-tuned using open annotated paraphrase datasets \cite{zhang2019paws}, it's important to note that there is no specific annotated paraphrasing data available for the IMDB and Yelp datasets. Therefore, the language models used in DP-Prompt are not explicitly fine-tuned for paraphrasing tasks related to these datasets. Consequently, the results obtained in Section \ref{sec:experiments} are likely to generalize to a large extent.

\noindent\textbf{Robustness of Evaluation Setup:}  Our evaluation methodology goes beyond relying solely on the Sentiment F1 score (Utility) or even comparing sentiment F1 score (Utility) with $\epsilon$ (theoretical privacy) as presented in \cite{meehan2022sentence}. Instead, we adopt a comprehensive approach that considers the merit of our proposed approach based on the sentiment F1 score in conjunction with empirical privacy demonstration against de-anonymization attacks. Simply copying reviews verbatim would yield a high author identification F1 score, failing to provide empirical privacy. Additionally, we consider not only static attack models but also stronger ones like adaptive attackers. Superficially altering reviews may deceive static attackers but would likely fail against adaptive attackers.

\section{Related Work}

The previous work on releasing private documents can be categorized into three approaches based on the level at which noise is added (See Table \ref{tab:summary} for concise summary). These approaches are:

\noindent \textbf{Word-level Approaches}: 
MadLib  \cite{feyisetan2020privacy} is a word-level mechanism that applies Laplace noise to word embeddings and maps them to the nearest word in the vocabulary, demonstrating the differential privacy (DP) guarantee of MadLib under the Euclidean metric. An extension of this approach involves using a regularized Mahalanobis metric instead \cite{xu2020differentially}. In contrast, the TEM mechanism utilizes the exponential mechanism to transform the privatization step into a selection problem \cite{carvalho2021tem}. Furthermore, there is a recent development known as CusText \cite{chen2023customized}, which focuses on developing customized mapping mechanisms for each individual word in the vocabulary \cite{chen2023customized}. All of these approaches are word-level mechanisms and have been shown to have significant limitations, such as their disregard for contextual information \cite{mattern2022limits}.



\noindent \textbf{Sentence-level Approaches}: Sentence-level mechanisms based on Sentence Transformer \cite{reimers2019sentence} were introduced in \cite{meehan2022sentence}. They proposed two approaches: one approach where noise is added to sentence embeddings, and another more complicated approach based on maximizing Tukey depth \cite{tukey1975mathematics, gilad-bachrach2012the}.

\noindent \textbf{Document-level Approaches}: A document-level Local Differential Privacy (LDP) mechanism was introduced, where GPT-2 is fine-tuned for a paraphrasing task \cite{mattern2022limits}. Our approach, DP-Prompt, draws inspiration from their work, but instead of resource-intensive fine-tuning, we use a zero-shot approach with pretrained models for efficient and effective generation of sanitized documents. Furthermore, the recently proposed DP-BART \cite{igamberdiev-habernal-2023-dp} employs BART \cite{lewis2020bart}, an encoder-decoder model. In DP-BART, noise is added to the encoder's output, and the decoder is fine-tuned to adapt to this noisy encoder output.

\noindent \textbf{Adversarial Methods:} Parallel to differentially private approaches, other techniques have been proposed that utilize Adversarial Learning \cite{shetty2018a4nt, quiring2019misleading} and Data Poisoning \cite{wang2022upton, jin2020bert}. However, these methods typically require access to a surrogate classifier. In contrast, our method is zero-shot, requiring neither fine-tuning nor access to a classifier.


\noindent \textbf{Differentially Private Training/Fine Tuning:} There is extensive research on differentially private training or fine-tuning of language models \cite{kerrigan2020differentially, li2021large, yu2021differentially, anil2022large, mattern2022differentially}. They aim to make language models resistant to various kinds of data leakage attacks \cite{carlini2019secret, carlini2021extracting, deng2021tag, balunoviclamp}. It is important to emphasize that this line of work is completely distinct from our own, as it focuses on training language models on private data, while our goal is to generate sanitized documents from private documents using pretrained language models.

\section{Conclusion}

This paper introduces DP-Prompt, a locally differentially mechanism called DP-Prompt that generates sanitized versions of private documents by prompting large language models to generate paraphrases. Notably, our method offers a simpler approach compared to existing methods. Through extensive experiments, we show that our approach achieves significantly improved utility compared to current methods for any required level of privacy. As the demand for on-device large language models (LLMs) continues to grow, our method emerges as a reliable safeguard for users' privacy and provides robust defense against de-anonymization attacks.

\section{Limitations}

In our study, we explored the initial step of harnessing large language models and zero-shot prompting to generate sanitized documents. While this approach effectively conceals the specific writing style of authors, there is still a potential risk of revealing explicit personal information, such as zip codes, bank details, or gender, especially when naively prompting a language model. This risk is particularly relevant in alternative text formats like messages or emails compared to online reviews.

For future work, an important direction would be to define a set of sensitive attributes and directly prompt the language model to replace these attributes with the identifier "X" while generating paraphrases. This approach would help improve the safeguarding of personal information. Additionally, it would be worthwhile to investigate the potential side effects of hallucination and the impact of different prompt templates on the generation of paraphrases, specifically within the context of the privacy-utility tradeoff. Additionally, more robust attacks that measure privacy leakage at the text level should be explored.



\bibliography{emnlp2023}
\bibliographystyle{acl_natbib}

\appendix

\section{Proof of Privacy}
\label{sec:appendix_proof}

\begin{theorem}
    \label{th:appLDP-Garuntee}
   Suppose the language model has not been pretrained on the private documents distribution $\mathcal{D}$. If the logits $\textbf{u} \in \mathbb{R}^{\abs{\mathcal{V}}}$ before the final softmax layer satisfy the condition $b_1 \leq u_i \leq b_2, \forall i \in [\abs{\mathcal{V}}]  $, and the \textup{DP-Prompt} run  with a temperature $T$ for generating $n$ tokens, then it can be proven that the generated output satisfies $(n(b_2-b_1)/T)$-\textup{LDP}.
\end{theorem}

\begin{proof}
   Let $\mathcal{D}$ and $\mathcal{D}'$ be any two documents, and $\textbf{u}$ and $\textbf{u}' \in \mathbb{R}^{\abs{\mathcal{V}}}$ be their corresponding logits. Let $v \in \mathcal{V}$, and $k$ be its index, with $u_{k}$ being its corresponding logit. We then have that,
    \begin{align*}
        \frac{\textup{Pr}[\mathcal{M} (\textup{D}) = v] }{\textup{Pr}[\mathcal{M}(\textup{D}') = v]} &= \dfrac{\frac{\exp(\frac{u_{k}}{T})}{\sum_{j=1}^{\abs{\mathcal{V}}} \exp(\frac{u_{j}}{T})}}{\frac{\exp(\frac{u_{k}'}{T})}{\sum_{j=1}^{\abs{\mathcal{V}}} \exp(\frac{u_{j}'}{T})}} \\
        &= \exp{\left(\frac{u_k - u_k'}{T}\right)} \frac{\sum_{j=1}^{\abs{\mathcal{V}}} \exp(\frac{u_{j}'}{T})}{\sum_{j=1}^{\abs{\mathcal{V}}} \exp(\frac{u_{j}}{T})} \\
        &\leq \exp{\left(\frac{b_2-b1}{T}\right)} \exp{\left(\frac{b_2-b1}{T}\right)}  \\
        & \leq \exp{\left(2(b_2-b_1)/T\right)}.
    \end{align*}
    Now by using sequential composition law of DP \cite{dwork2006calibrating}, we can set $\epsilon = \left(2n(b_2-b_1)/T\right)$ to conclude the proof. 
\end{proof}

\section{Code}
\label{sec:appendix_code}

In this section, we showcase the code for DP-Prompt implemented using the Hugging Face library \cite{wolf2019huggingface}. The code can be found in Listing \ref{lst:dp-prompt-code}.

\begin{figure*}[t]
\begin{lstlisting}[caption={DP Prompt Code using Hugging Face's Transformers Library},label={lst:dp-prompt-code}, numbers=none]
from transformers import LogitsProcessor, LogitsProcessorList

class ClippedLogitsProcessor(LogitsProcessor):
    def __init__(self, min_tensor, max_tensor):
        self.min_tensor = min_tensor
        self.max_tensor = max_tensor

    def __call__(self, input_ids, scores):
        clipped_logits = torch.clamp(scores, self.min_tensor, self.max_tensor)
        return clipped_logits

def prompt_template_fn(private_doc):
    prompt = f"Document: {private_doc}\nParaphrase of the document:"
    return prompt
    
def dp_prompt(
    private_doc, model, tokenizer, min_tensor, max_tensor, temp, new_tokens
):
    logits_processor = LogitsProcessorList(
        [ClippedLogitsProcessor(min_tensor, max_tensor)]
    )
    private_prompt = prompt_template_fn(private_doc)
    input_ids = tokenizer.encode(private_prompt, return_tensors="pt")
    output = model.generate(
        input_ids,
        do_sample=True,
        top_k=0,
        top_p=1.0,
        temperature=temp,
        max_new_tokens=new_tokens,
        logits_processor=logits_processor
    )
    sanitized_doc = tokenizer.decode(output[0][0], skip_special_tokens=True)
    return sanitized_doc
\end{lstlisting}
\end{figure*}

\section{Implementation Details and Additional Experiments}
\label{sec:appendix_imp_details}

\subsection{Word-Level Mechanisms}

For all word-level mechanisms, we utilize 50-dimensional glove embeddings \cite{pennington2014glove}, following the approach of Mattern et al. \cite{mattern2022limits}. The projection step, which requires approximate nearest neighbor search, is performed using an Annoy indexer with 500 trees.

\subsection{Sentence-Level Mechanisms}

Both the Truncated-Laplace and Deep-Candidate mechanisms require additional publicly available data. In the case of Truncated-Laplace, this data is used to determine truncated boundaries, while for Deep-Candidate, it is used to train the sentence recoder and obtain the output.

We randomly sample 5,000 documents from both the IMDB and Yelp reviews, which are not part of the data used for privacy-utility experiments.

The sentence recoder architecture is trained with three layers of MLP, incorporating dropout and selecting the best model. We choose 50 clusters for sentence recoding and employ 100 random projections to estimate the approximate Tukey depth.

Sentence embeddings \cite{reimers2019sentence} of dimension 768 are obtained from \cite{song2020mpnet}.

\subsection{Document-Level Mechanisms}

For the paraphrasing mechanism, we fine-tune the gpt2-xl(1.5b) parameter model on the PAWS dataset \cite{zhang2019paws}. The training set is constructed by combining the train and val sets, and the test set is used for validation to save the best model. The training set consists of 25,368 examples, and the validation set consists of 3,536 examples. We follow the procedure outlined in \cite{mattern2022limits}, which builds upon \cite{witteveen2019paraphrasing}.

To set clip thresholds in Section \ref{subsec:clipping}, we employ a process similar to Truncated-Laplace \cite{meehan2022sentence} with a slight difference. While \cite{meehan2022sentence} calculates the $75\%$ quantile and utilizes it as the clipping threshold, we calculate the minimum and maximum values.Both Truncated-Laplace and Dp-Prompt employ the exact same additional data. The resulting min and max clip threshold vector is then used to clip the logits before scaling them by temperature.

\subsection{Static and Adaptive Attacker Architecture}

 Word level and document level output documents, to simulate embedding-level attacker  we employ sentence transformer \texttt{all-mpnet-base-v2} \cite{reimers2019sentence, song2020mpnet} to convert sanitized document to sanitized embedding. For sentence level, directly sanitized embeddings are used. The embedding-level attackers employ a three-layer MLP with a hidden dimension of 768.  We use the ReLU activation function and incorporate dropout of 0.5. The models are trained for 50 epochs with a batch size of 32, using the Adam optimizer with a StepLR learning scheduler. The initial learning rate is set to $10^{-3}$, and the gamma value is 0.95. 

The text-level attackers use \texttt{bert-base-cased} and fine-tune it for 3 epochs with a batch size of 16, using the AdamW optimizer with a linear scheduler and a starting learning rate of $5 \times 10^{-5}.$

\subsection{Code and Reproducibility}

The code will be publicly released. Considering the reproducibility challenges associated with closed APIs \cite{pozzobon2023challenges}, we also plan to release the paraphrased documents that were generated using ChatGPT.

.

\begin{table*}[htbpt!]
\centering
\resizebox{1.3\columnwidth}{!}{%
\begin{tabular}{ccccccccccc}
\hline
\multicolumn{11}{c}{\textbf{BERT Score}}                                                                                                                  \\ \hline
\multicolumn{1}{c|}{}                                & \multicolumn{5}{c|}{\textbf{IMDB}}                         & \multicolumn{5}{c}{\textbf{Yelp}}     \\ \hline
\multicolumn{1}{c|}{} &
  \textbf{t=1.0} &
  \textbf{t=1.25} &
  \textbf{t=1.5} &
  \textbf{t=1.75} &
  \multicolumn{1}{c|}{\textbf{t=2.0}} &
  \textbf{t=1.0} &
  \textbf{t=1.25} &
  \textbf{t=1.5} &
  \textbf{t=1.75} &
  \textbf{t=2.0} \\  \hline 
\multicolumn{1}{c|}{\textbf{Chat GPT}} & 0.882 & 0.879 & 0.863 & 0.805 & \multicolumn{1}{c|}{0.765} & 0.89 & 0.887 & 0.876 & 0.83 & 0.777 \\
\end{tabular}%
}
\caption{Bert Score for ChatGPT for different sampling temperatures}
\label{tab:chatgpt_bert_score}
\end{table*}

\begin{table*}[htbp!]
\centering
\resizebox{1.4\columnwidth}{!}{%
\begin{tabular}{ccccccccccc}
\hline
\multicolumn{11}{c}{\textbf{BERT Score}}                                                                                                                  \\ \hline
\multicolumn{1}{c|}{}                                & \multicolumn{5}{c|}{\textbf{IMDB}}                         & \multicolumn{5}{c}{\textbf{Yelp}}     \\ \hline
\multicolumn{1}{c|}{} &
  \textbf{t=0.75} &
  \textbf{t=1.0} &
  \textbf{t=1.25} &
  \textbf{t=1.5} &
  \multicolumn{1}{c|}{\textbf{t=1.75}} &
  \textbf{t=0.75} &
  \textbf{t=1.0} &
  \textbf{t=1.25} &
  \textbf{t=1.5} &
  \textbf{t=1.75} \\ \hline
\multicolumn{1}{c|}{\textbf{GPT-2 (xl) (fine tuned)}} & 0.838 & 0.822 & 0.794 & 0.762 & \multicolumn{1}{c|}{0.748} & 0.852 & 0.837 & 0.804 & 0.765 & 0.750 \\
\multicolumn{1}{c|}{\textbf{T5 (xl)}}                & 0.831 & 0.812 & 0.790 & 0.775 & \multicolumn{1}{c|}{0.763} & 0.834 & 0.817 & 0.797 & 0.782 & 0.769 \\
\multicolumn{1}{c|}{\textbf{Stablelm-Base (3b)}}     & 0.814 & 0.805 & 0.785 & 0.762 & \multicolumn{1}{c|}{0.749} & 0.835 & 0.817 & 0.789 & 0.757 & 0.747 \\
\multicolumn{1}{c|}{\textbf{Stablelm-Base (7b)}}     & 0.813 & 0.803 & 0.779 & 0.759 & \multicolumn{1}{c|}{0.748} & 0.835 & 0.819 & 0.788 & 0.757 & 0.746 \\
\multicolumn{1}{c|}{\textbf{Flan T5 (xl)}}           & 0.843 & 0.823 & 0.792 & 0.762 & \multicolumn{1}{c|}{0.750} & 0.849 & 0.830 & 0.801 & 0.769 & 0.753 \\
\multicolumn{1}{c|}{\textbf{Stablelm Tuned (3b)}}    & 0.846 & 0.830 & 0.795 & 0.757 & \multicolumn{1}{c|}{0.743} & 0.849 & 0.836 & 0.800 & 0.760 & 0.746 \\
\multicolumn{1}{c|}{\textbf{Stablelm Tuned (7b)}}     & 0.854 & 0.839 & 0.800 & 0.757 & \multicolumn{1}{c|}{0.743} & 0.858 & 0.845 & 0.806 & 0.761 & 0.747
\end{tabular}%
}
\caption{Bert Score for Open Source Models for different sampling temperatures}
\label{tab:small_bert_score}
\end{table*}

\subsection{Expected F1 score of random classifier} \label{subsec:random_classifier}

Let $p_i$ represent the fraction of documents with label $y_i$, where $i$ ranges from $0$ to $\ell-1$. A uniformly random classifier predicts class $y_i$ with a probability of $1/\ell$. Based on this setup, we can derive the following metrics:

\begin{align*}
    \text{True Positives(TP)} &= \frac{p_i}{\ell} \\ 
    \text{True Negatives(TN)} &= \frac{(1-p_i) (\ell -1)}{\ell} \\ 
    \text{False Positives(FP)} &= \frac{1-p_i}{\ell}  \\ 
    \text{False Negatives(FN)} &=  \frac{p_i (\ell -1)}{\ell} 
\end{align*}

From these metrics, we can calculate the F1 Score for sentiment analysis, which is a binary classification task, as follows:

\begin{align*}
    \text{F1 Score} &= \frac{\text{TP}}{\text{TP} + \frac{1}{2} \left(\text{FP} + \text{FN}\right)}= \frac{2p_1}{1 + p_{1} \ell}
\end{align*}

For author identification, which is a multi-class classification task, we utilize the F1 Score with a macro average. In the case of a random classifier, the expected F1 score can be calculated as:

\begin{align*}
    \text{F1 Score}_{\text{macro avg}} &= \frac{1}{\ell} \sum_{i=0}^{\ell-1} \frac{2p_i}{1 + p_{i} \ell}
\end{align*}

By employing this approach, we can effectively evaluate the performance of a random classifier for author identification in terms of the F1 Score.

\section{Additional Results}

\subsection{Comparing 50 and 300-Dimensional Glove Embeddings for Word-Level Mechanism}

Note that in Section \ref{subsec:baseline}, we used 50-dimensional Glove embeddings (\texttt{glove-wiki-gigaword-50}) \cite{pennington2014glove} for Madlib, Mahalanobis, and TEM mechanisms.  In this section, we demonstrate that there is no significant benefit to using 300-dimensional glove embeddings (\texttt{glove-wiki-gigaword-300}). We show this for IMDB dataset against embedding-level attackers. Results are show in Figure \ref{fig:embedding-300-50}.

\begin{figure}[t!]
  \centering
  \begin{subfigure}[b]{0.48\linewidth}
    \centering
    \includegraphics[width=\linewidth]{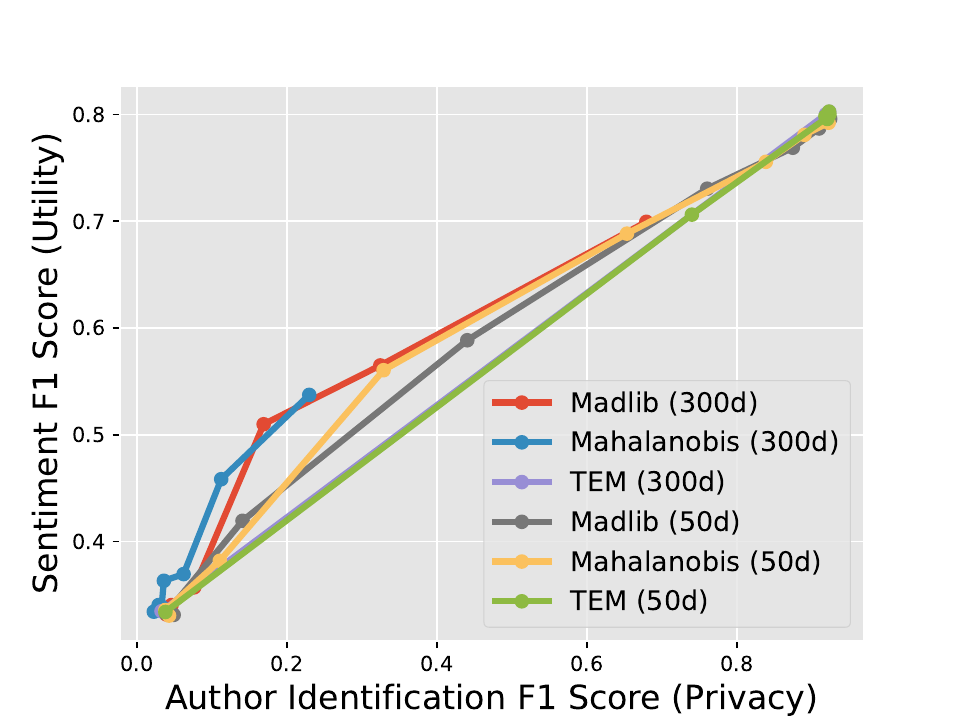}
    \subcaption{IMDB (static)}
    \label{fig:embedding_300vs50_imdb_static}
  \end{subfigure}
  \begin{subfigure}[b]{0.48\linewidth}
    \centering
    \includegraphics[width=\linewidth]{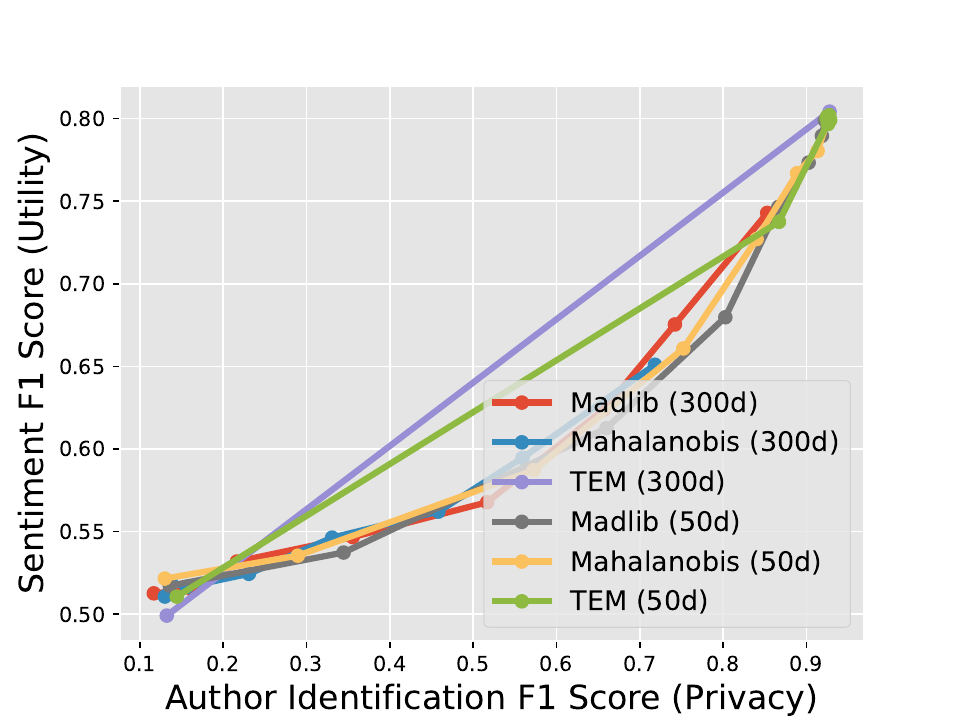}
    \subcaption{IMDB (adaptive)}
    \label{fig:embedding_300vs50_imdb_adaptive}
  \end{subfigure}
  
  \caption{Comparision of \texttt{glove-wiki-gigaword-50} with \texttt{glove-wiki-gigaword-300} for word level mechanism (Madlib, Mahalonobis, TEM) for IMDB dataset against embedding level attackers.}
  \label{fig:embedding-300-50}
\end{figure}

\subsection{BERT score for paraphrasing models}
We use BERTScore, which has been demonstrated to correlate with human judgments, in conjunction with the RoBerta model \cite{liu2019roberta} (\texttt{roberta-large}) to assess the similarity between a review and its paraphrase. The results for ChatGPT are presented in Table \ref{tab:chatgpt_bert_score}, while those for the open-source model can be found in Table \ref{tab:small_bert_score}.

\section{Sample Illustration of paraphrases}

In this section, we present a comparison of paraphrases generated from ChatGPT and Stablelm at various temperatures, as illustrated in Figure \ref{fig:chat_gpt_vs_stablelm_tuned}. It is evident from the results that ChatGPT produces higher-quality paraphrases compared to Stablelm-Tuned (7B).

\begin{figure*}
\includegraphics[width=\textwidth]{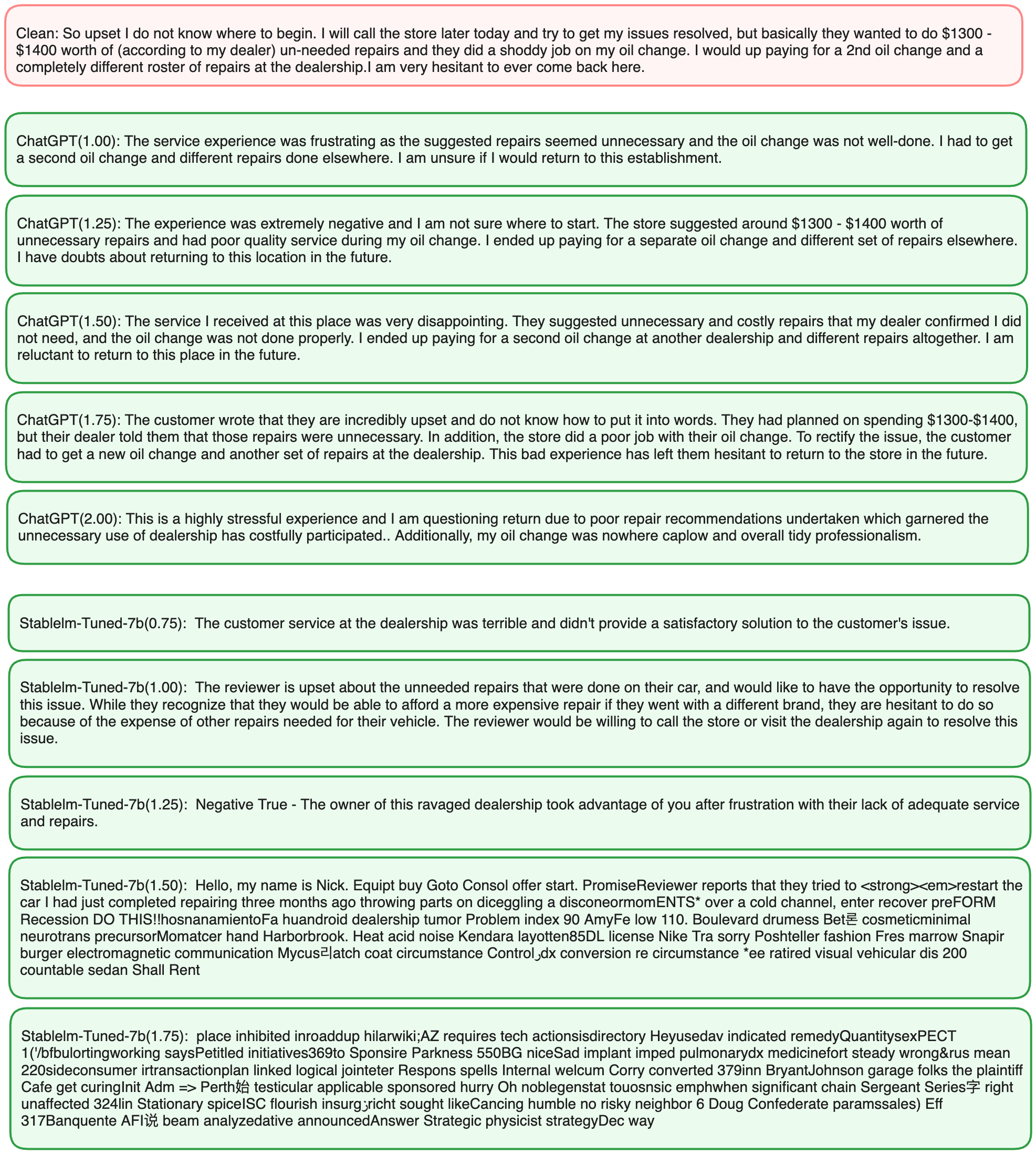}
\caption{Comparison of paraphrase results between ChatGPT (gpt-3.5) and Stablelm-Tuned-7B at different temperatures. ChatGPT consistently generates more readable and superior quality text at higher temperatures compared to Stablelm-Tuned-7B.}
\label{fig:chat_gpt_vs_stablelm_tuned}
\end{figure*}
\end{document}